%% file: sample-authordraft.tex
\documentclass[sigconf]{acmart}

\usepackage[utf8]{inputenc}
\usepackage[T1]{fontenc}
\usepackage{hyperref}
\usepackage{url}
\usepackage{booktabs}
\usepackage{amsfonts}
\usepackage{nicefrac}
\usepackage{microtype}
\usepackage{todonotes}
\usepackage{multirow}
\usepackage{wrapfig}
\usepackage{paralist}
\usepackage{subfigure}
\usepackage{amsmath,amsthm}
\usepackage{etoolbox}
\usepackage{array}
\usepackage{tikz}
\usepackage{pgfplots}
\pgfplotsset{compat=1.17}
\usetikzlibrary{shapes,arrows,arrows.meta, positioning,shapes.misc,decorations.markings,pgfplots.groupplots}
\usepackage{adjustbox}
\usepackage{hyperref}

\usepackage{booktabs}
\usepackage{lipsum,multicol}
\usepackage{enumitem}
\usepackage{appendix}

\newtheorem{theorem}{Theorem}
\newtheorem{corollary}{Corollary}[theorem]

\everypar{\looseness=-1} %
\title{A Framework for Cluster and Classifier Evaluation in the Absence of Reference Labels}

\author{Robert J. Joyce}
\affiliation{\institution{Booz Allen Hamilton\\University of Maryland, Baltimore County}\country{USA}}
\email{joyce\_robert2@bah.com}

\author{Edward Raff}
\affiliation{\institution{Booz Allen Hamilton\\University of Maryland, Baltimore County}\country{USA}}
\email{raff\_edward@bah.com}

\author{Charles Nicholas}
\affiliation{\institution{University of Maryland, Baltimore County}\country{USA}}
\email{nicholas@umbc.edu}

\copyrightyear{2021}
\acmYear{2021}
\setcopyright{acmlicensed}
\acmConference[AISec '21] {Proceedings of the 14th ACM Workshop on Artificial Intelligence and Security}{November 15, 2021}{Virtual Event, Republic of Korea.}
\acmBooktitle{Proceedings of the 14th ACM Workshop on Artificial Intelligence and Security (AISec '21), November 15, 2021, Virtual Event, Republic of Korea}
\acmPrice{15.00}
\acmISBN{978-1-4503-8657-9/21/11}
\acmDOI{10.1145/3474369.3486867}

\begin{document}


\begin{abstract}
In some problem spaces, the high cost of obtaining ground truth labels necessitates use of lower quality reference datasets. It is difficult to benchmark model performance using these datasets, as evaluation results may be biased. We propose a supplement to using reference labels, which we call an approximate ground truth refinement (AGTR). Using an AGTR, we prove that bounds on specific metrics used to evaluate clustering algorithms and multi-class classifiers can be computed without reference labels. We also introduce a procedure that uses an AGTR to identify inaccurate evaluation results produced from datasets of dubious quality. Creating an AGTR requires domain knowledge, and malware family classification is a task with robust domain knowledge approaches that support the construction of an AGTR. We demonstrate our AGTR evaluation framework by applying it to a popular malware labeling tool to diagnose over-fitting in prior testing and evaluate changes whose impact could not be meaningfully quantified under previous data.
\end{abstract}


\maketitle

\section{Introduction} \label{sec:introduction}

Evaluating the performance of a real-world malware classifier, or a system which can cluster malware by family, requires large amounts of labeled data. Simultaneously, human-annotated labels for malware are uniquely expensive - unable to be crowdsourced via services such Mechanical Turk \cite{mechanicalturk} and requiring highly skilled knowledge labor. Existing reference datasets for evaluating the performance of these models are insufficient, so the focus of our research is to ask the following: 

\begin{enumerate}
\item Can robust statements about a system’s overall performance be made using only large, unlabeled corpora?
\vspace*{4pt}
\item Can the impact of design changes upon such a system's performance be quantified using only large, unlabeled corpora?
\end{enumerate}

Our work answers both cases in the affirmative, via a mechanism that we term an approximate ground truth refinement (AGTR). An AGTR is an incomplete clustering of a dataset by ground truth reference label, in which data points belonging to the same cluster share an (unknown) ground truth reference label, but the relationship between data points in different clusters is unknown. Since an AGTR of a dataset informs a subset of data point relationships, it can be used to glean some knowledge about model performance, even if the reference labels themselves are not known. In practice, creation of such a clustering requires domain knowledge effort that can group "alike" data points with high fidelity, and users must account for a minimal rate of error inherent to this approach. AGTRs become increasingly effective at model evaluation the more similar they are to the (unknown) ground truth. They are best employed in problem spaces with a large number of classes and where determining the label of an individual data points is difficult, but grouping similar data points is more tractable.

Using an AGTR with no more than $\hat{\epsilon}$ errors, we provide provable bounds on specific metrics used to evaluate any system that can group samples into sub-populations (e.g, traditional clustering, or by treating predicted class labels as sub-populations). Even if an AGTR’s error rate is not well-specified, the relative rankings produced by an AGTR are invariant to $\hat{\epsilon}$, still allowing for robust conclusions. We discuss the properties an ideal AGTR should have and demonstrate a method for constructing an AGTR from a corpus of malware in the Windows Portable Executable (PE) file format using a well-known and battle-tested metadata hashing algorithm. AGTRs can also be used to test the quality of noisy reference labels and the validity of evaluation results, and we use these capabilities to provide evidence that the original benchmark results for the seminal AVClass malware classifier \cite{avclass} may be over-fit. Finally, AGTRs can be used to compare the performance of intrinsically similar classifiers and clustering methods. In a case study, we evaluate the performance impact of various modifications to AVClass, which could not be previously elucidated due to a lack of precise labels.

\subsection{Related Work}
\label{sec:relatedWork}

While there is increasing focus on the accurate quantification, reliability, and reproducibility of machine learning \cite{Bouthillier2021,marie-etal-2021-scientific,pmlr-v119-engstrom20a,Barz2019,Musgrave2020,10.1145/3383313.3412489,pmlr-v97-bouthillier19a,Raff2019_quantify_repro}, few works have considered the use of unlabeled data to further evaluate models on a labeled task. The most similar work to ours is \citet{Deng2021}, which proposes AutoEval, a method for estimating the accuracy of a multi-class classifier without a labeled reference dataset by using the training set and dataset-level feature statistics. The authors implemented this by generating a synthetic meta-dataset of images by applying various transformations to each image in the original dataset, then training a regression model to predict accuracy. Other related work includes \citet{Jiang2021}, which found that the error of a neural network can be estimated by observing disagreement rates of unlabeled data and \citet{novak2019}, which used unlabeled data to augment the automatic evaluation of surface cohesion of essays written in Czech. Finally, \citet{oliver2018} challenged existing practices for evaluating self-supervised learning models trained using unlabeled data. We are aware of no prior work which provably bounds the performance of a multi-class classifier or clustering algorithm using only unlabeled data. Other malware detection concerns around the proper evaluation of false positive rates also exist \cite{Nguyen2021}, but are beyond the scope of our study.

\subsection{Clustering Terminology and Metrics}
\label{cluster-metrics}

A variety of external validity indices are used to evaluate the performance of clustering algorithms using labeled reference data. In this paper, we focus on the metrics precision and recall. Historically, precision and recall have been used for evaluating the performance of information retrieval systems \cite{kent}. \citet{bayer} introduced alternate definitions of precision and recall that function as cluster validity indices, for the purpose of evaluating malware family clustering. Precision penalizes the presence of impure clusters, \emph{i.e.}, clusters containing data points belonging to separate reference clusters. Conversely, recall penalizes instances in which data points belonging to the same reference cluster are not grouped into the same predicted cluster \cite{li}. Precision and recall can be thought of as somewhat analogous to homogeneity and completeness but with minor differences. Let $M$ be a labeled reference dataset consisting of $m$ unique data points. Let $C = \{C_{i}\}_{1\leq i \leq c}$ and $D = \{D_{j}\}_{1\leq j \leq d}$ each partition $M$, where $C$ is the predicted clustering of $M$, $D$ is a clustering of $M$ by reference label (i.e. the desired cluster assignments, which we refer to as a \emph{reference clustering}), and $c$ and $d$ are the number of clusters in $C$ and $D$ respectively. The formal definitions of precision and recall are provided below:

\theoremstyle{definition}
\begin{definition} \label{def:precision}
$Precision(C, D) = \frac{1}{m} \;\sum\limits_{i=1}^{c} |C_{i} \cap D_{f(i)}|$ 
\end{definition}

\begin{definition} \label{def:recall}
$Recall(C, D) = \frac{1}{m} \;\sum\limits_{j=1}^{d} |C_{g(j)} \cap D_{j}|$
\end{definition}

Definitions \ref{def:precision} and \ref{def:recall} require the cluster mapping functions $f : \{1...c\} \mapsto \{1...d\}$ and $g : \{1...d\} \mapsto \{1...c\}$. These functions can be given any definition, but when used to evaluate clusterings they are typically defined as \cite{li}:

\vspace*{0.1cm}
\begin{center}
$f(i) = $ arg$\max\limits_{j}|C_{i} \cap D_{j}|$ and $g(j) = $ arg$\max\limits_{i}|C_{i} \cap D_{j}|$
\end{center}

Under these definitions, $f$ and $g$ map each predicted cluster to the reference cluster for which there is maximal overlap and vice versa. In Figure \ref{fig:CD}, $\textit{Precision(C, D)} = 0.75$ and  $\textit{Recall(C, D)} = 0.5$.

\begin{figure}[!t]
    \centering
    \adjustbox{max width=\columnwidth}{%
        \input{AGTR_all_diagram}
    }
    \caption{
    Four partitions of a hypothetical dataset with eight data points. The predicted clusters (left, "$C$") would ideally be evaluated using ground truth reference clusters (second from left, "$D$"). A GTR (second from right, "$R$") informs a subset of the data point relationships in $D$ (e.g., 5 and 7 ($R_4$) must belong to the same reference cluster, but without $D$ it is unknown whether the members of $R_4$ and $R_5$ share a reference cluster). An AGTR (right, "$\hat{R}$") is a GTR with $\epsilon$ errors. $\hat{R}_2$ incorrectly groups data point 2 with 3 and 4, so $\hat{R}$ has $\epsilon$ = 1 errors. If data point 2 is removed from $\hat{R}_{2}$, $\hat{R}$ would become a GTR.
    }
    \label{fig:CDRhatR}
    \label{fig:CD}
    \label{fig:DR}
    \label{fig:DRhat}
\end{figure}
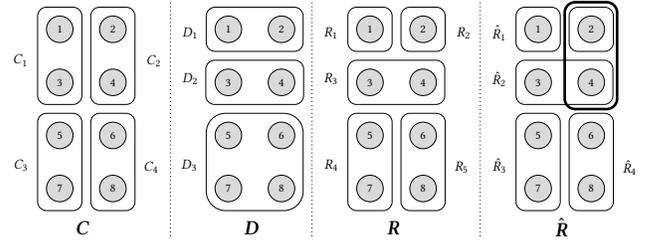
\vspace*{-2pt}

\subsection{Classification Terminology and Metrics}
\label{sec:classifier-metrics}

Although precision and recall were originally intended to be used as clustering validity indices, their definitions also permit evaluation of multi-class classifiers \cite{li}. Specifically, these metrics are useful when the set of labels used by the classifier do not match the set of labels used by the reference dataset. This is an issue, for example, with malware family naming - different vendors often refer to the same family of malware using different terminology. The functions $f$ and $g$ enable manually-defined mappings between label sets, but they are most commonly assigned the definitions listed in Section \ref{cluster-metrics}. To evaluate a classifier using precision and recall, the predicted labels must be converted into a predicted clustering $C$ and the reference labels must be converted into a reference clustering $D$. This is achieved by assigning all data points which share a label to the same cluster. Note that precision and recall cannot evaluate whether a classifier has predicted the correct labels for data points; rather, they quantify the structural similarity between $C$ and $D$ (which were generated from the from the predicted labels and reference labels respectively).

Additionally, Li \emph{et al.} showed that the accuracy of a classifier can be computed as a special case of precision and recall. Accuracy measures the proportion of data points whose reference labels were correctly predicted by the classifier. In order to compute accuracy, the set of labels used by the classifier must be identical to the set of labels used by the reference dataset, e.g. $c$ = $d$, and $f$ and $g$ are defined as the identity function \cite{li}.

\begin{center}
$\forall i$, $1 \leq i \leq c$, $f(i) = i$ and $g(i) = i$.
\end{center}

The formal definition of accuracy is given by \cite{li}:

\begin{definition}
\label{def:accuracy}
If $f$ and $g$ are the identity function, \textit{Accuracy(C, D)} = \textit{Precision(C, D)} = \textit{Recall(C, D)}
\end{definition}

In Figure \ref{fig:CD}, $\textit{Recall(C, D)}$ cannot be computed because there is not a one-to-one mapping between the predicted clusters $C$ and the reference clusters $D$.

\section{Approximate Ground Truth Refinements} \label{sec:computing_gtr_agtr}

In this section, we introduce the concept of a ground truth refinement (GTR) and show that a GTR can be used to find provable bounds on precision, recall, and accuracy. In practice, when constructing a GTR from a dataset we assume that a small number of errors occur. We call an imperfect GTR an approximate ground truth refinement (AGTR). We show that bounds on these evaluation metrics can still be proven using an AGTR if errors in the AGTR construction process are properly accounted for. Finally, we propose a framework that uses an AGTR to evaluate clustering algorithms and multi-class classifiers when satisfactory reference data is unavailable. All proofs are given in \autoref{sec:proofs}.

\subsection{Set Partition Refinements}
\label{sec:refinements}

A key element of this work is the concept of a set partition refinement. Suppose $R$ and $S$ are two partitions of the same set $M$. $R$ is a \emph{refinement} of $S$ if each set within $R$ is a subset of some set in $S$ \cite{paige1987}.

\begin{definition} \label{def:refinement}
If\, $\forall R_{k} \in R$, $\exists S_{j} \in S$ s.t. $R_{k} \subseteq S_{j}$ then $R$ is a set partition refinement of $S$.
\end{definition}

Set partition refinements can also be considered from an alternate perspective. If $R$ is a refinement of $S$, then $S$ can be constructed by iteratively merging sets within $R$. Specifically, each set $S_{j} \in S$ is equivalent to the union of some unique set of sets within $R$. 

\begin{enumerate}[label=\normalfont{Property \arabic*.},itemindent=35pt,align=left,leftmargin=0em,font=\textbf]
  \item $\forall S_{j} \in S$, $\exists!Q_{j}=\{Q_{j\ell}\}_{1 \leq \ell \leq q_{j}}$ s.t. $S_{j} = \bigcup\limits_{\ell=1}^{q_{j}}Q_{j\ell}$ and $\forall Q_{j\ell} \in Q_j$, $Q_{j\ell} \in R$
\end{enumerate}
\vspace*{5pt}

In Section \ref{sec:GTR}, we use Definition \ref{def:refinement} and Property 1 to prove properties of ground truth refinements, which are a type of set partition refinement.

\subsection{Ground Truth Refinements}
\label{sec:GTR}

A \emph{ground truth refinement} (GTR) of a dataset is a clustering where all data points in a cluster are members of the same ground truth reference cluster. Importantly, the opposite is not necessarily true, as data points in the same reference cluster can belong to different clusters in the GTR.

\begin{definition}
\label{def:gtr}
If $D$ is a ground truth reference clustering and $R$ is a refinement of $D$, then $R$ is a ground truth refinement.
\end{definition}

Recall that for a dataset $M$, $C = \{C_{i}\}_{1\leq i \leq c}$ is the predicted clustering and $D = \{D_{j}\}_{1\leq j \leq d}$ is the reference clustering. Let $D$ have ground truth confidence and let $R = \{R_{k}\}_{1\leq k \leq r}$ be a GTR of $D$. Since $R$ partitions $M$, it is possible to compute the precision and recall of $C$ with respect to $R$ rather than $D$. An important trait of a GTR is that it does not require reference labels. Since $R$ is unlabeled, we map each predicted cluster in $C$ to the cluster in the $R$ for which there is maximal overlap and vice versa. Let the functions for mapping between the predicted clusters and the GTR $f' : \{1...c\} \mapsto \{1...r\}$ and $g' : \{1...r\} \mapsto \{1...c\}$ be defined as:

\vspace*{0.1cm}
\begin{center}
$f'(i) = $ arg$\max\limits_{k}|C_{i} \cap R_{k}|$ and $g'(k) = $ arg$\max\limits_{i}|C_{i} \cap R_{k}|$
\end{center}

Using Definition \ref{def:refinement}, we prove that the precision of a clustering algorithm or multi-class classifier computed using the ground truth reference clustering is bounded below by its precision computed using a GTR. Similarly, using Property 1, we prove that recall computed using a GTR is always an upper bound on recall computed using the ground truth reference clustering. Because accuracy, precision, and recall are all equivalent in a special case (e.g. there is a one-to-one mapping between the predicted and reference clusterings), we prove that recall computed using a GTR is also always an upper bound on the accuracy of a classifier. These bounds provide the foundation for evaluating clustering algorithms and multi-class classifiers using an AGTR.

\setcounter{theorem}{0}
\begin{theorem} \label{thm:precisiongtr}
\textit{Precision(C, R)} $\leq$ \textit{Precision(C, D)}
\end{theorem}

\begin{theorem} \label{thm:recallgtr}
\textit{Recall(C, R)} $\geq$ \textit{Recall(C, D)}
\end{theorem}

\begin{corollary} \label{cor:accuracygtr}
\textit{Recall(C, R)} $\geq$ \textit{Accuracy(C, D)}
\end{corollary}

\subsection{Approximate Ground Truth Refinements}
\label{sec:AGTR}

Unfortunately, (with the exception of a trivial case discussed in \autoref{sec:trivial}), it is impossible to determine whether or not a clustering is a GTR without using the ground truth reference clustering as confirmation. Because we intend for GTRs to be used when satisfactory reference datasets are not available, this is problematic. To circumvent this issue, when attempting to construct a GTR we assume that the resulting clustering is very similar to a GTR, but has a small number of data points $\epsilon$ which violate the properties of a refinement. We call such a clustering an \emph{approximate ground truth refinement} (AGTR).

\setcounter{theorem}{5}
\begin{definition}
\label{def:agtr}
If $R$ is a ground truth refinement and $\hat{R}$ can be made equivalent to $R$ by correcting the cluster membership of $\epsilon$ data points, then $\hat{R}$ is an approximate ground truth refinement.
\end{definition}

Consider an AGTR $\hat{R}$ with $\epsilon$ erroneous data points. Even without knowing which data points must be corrected to transform $\hat{R}$ into a GTR, we can again derive bounds on \textit{Precision(C, D)} as well as upper bounds on \textit{Recall(C, D)} and \textit{Accuracy(C, D)} using $\hat{R}$ and $\epsilon$. To do this, we first show that the precision and recall change in predictable ways when a reference clustering is modified. Let $S$ be an arbitrary partition of a dataset $M$ with $m$ data points and let $\hat{S}$ be identical to $S$ but with a single data point belonging to a different cluster. When the precision and recall of $C$ are measured with respect to $S$ and $\hat{S}$, the resulting metric values share the following relationship:

\setcounter{theorem}{2}
\begin{theorem} \label{thm:precision1err}
$\lvert$\textit{Precision(C, S)} $-$ \textit{Precision(C, $\hat{S}$)}$\rvert$ $\leq \frac{1}{m}$
\end{theorem}

\begin{theorem} \label{thm:recall1err}
$\lvert$\textit{Recall(C, S)} $-$ \textit{Recall(C, $\hat{S}$)}$\rvert$ $\leq \frac{1}{m}$
\end{theorem}

Theorems \ref{thm:precision1err} and \ref{thm:recall1err} show that precision and recall can vary by up to $\pm \frac{1}{m}$ when the cluster membership of a single data point in the reference clustering is changed. Therefore, if $\epsilon$ cluster labels in $\hat{R}$ are erroneous, the difference between $\textit{Precision(C, R)}$ and $\textit{Precision(C, $\hat{R}$)}$ as well as between $\textit{Recall(C, R)}$ and $\textit{Recall(C, $\hat{R}$)}$ is at most $\pm \frac{\epsilon}{m}$

\setcounter{theorem}{3}
\begin{corollary} \label{thm:precisionerr}
$\lvert$\textit{Precision(C, R)} $-$ \textit{Precision(C, $\hat{R}$)}$\rvert$ $\leq \frac{\epsilon}{m}$
\end{corollary}

\setcounter{theorem}{4}
\setcounter{corollary}{0}
\begin{corollary} \label{thm:recallerr}
$\lvert$\textit{Recall(C, R)} $-$ \textit{Recall(C, $\hat{R}$)}$\rvert$ $\leq \frac{\epsilon}{m}$
\end{corollary}

These relationships require knowledge of the exact value of $\epsilon$, which is again not possible to determine without knowing the ground truth reference clustering. Since the purpose of an AGTR is to be used when an adequate reference clustering is unavailable, this presents a problem. Our solution is to select some value $\hat{\epsilon}$ with the belief that $\hat{\epsilon} \geq \epsilon$. We show that if this belief is true, the bounds on precision, recall, and accuracy are valid.

\begin{theorem} \label{thm:precisionagtr}
If $\hat{\epsilon} \geq \epsilon$ then \textit{Precision(C, $\hat{R}$)} $-$ $\frac{\hat{\epsilon}}{m}$ $\leq$ \textit{Precision(C, D)}
\end{theorem}

\begin{theorem} \label{thm:recallagtr}
If $\hat{\epsilon} \geq \epsilon$ then \textit{Recall(C, $\hat{R}$)} + $\frac{\hat{\epsilon}}{m}$ $\geq$ \textit{Recall(C, D)}
\end{theorem}

\begin{corollary}
\label{cor:accuracyagtr}
If $\hat{\epsilon} \geq \epsilon$ then \textit{Recall(C, $\hat{R}$)} + $\frac{\hat{\epsilon}}{m}$ $\geq$ \textit{Accuracy(C, D)}
\end{corollary}

They allow the bounds on the precision, recall, and accuracy of a clustering algorithm or a multi-class classifier to be computed without reference labels.

\subsection{Estimating Errors in an AGTR}
\label{sec:AGTRerror}

We emphasize that the evaluation metric bounds from Theorem \ref{thm:precisionagtr}, Theorem \ref{thm:recallagtr}, and Corollary \ref{cor:accuracyagtr} only hold if errors during AGTR construction are accounted for properly, \emph{i.e.} $\hat{\epsilon} \geq \epsilon$. Selecting a satisfactory value of $\hat{\epsilon}$ for an AGTR is a matter of epistemic uncertainty and is a topic of  future work. Determining the approximate error rate of a process used to construct an AGTR will likely require some guesswork, as a quality reference dataset is presumably unavailable. Domain experts should model their uncertainty about the AGTR construction method's error rate and choose a value of $\hat{\epsilon}$ that they believe exceeds the number of errors with very high confidence. In \autoref{sec:pehash}, we provide an example of how to evaluate the error rate of an AGTR in order to select a judicious value of $\hat{\epsilon}$.

\subsection{Properties of AGTRs and Resulting Bounds}
\label{sec:IdealAGTR}
Although we have proven that it is possible to compute bounds on precision, recall, and accuracy using an AGTR, we have not yet proposed any techniques for constructing an AGTR from a dataset. Just as there is no universal method which can be used to label a reference dataset, there is no singular approach which can be used for general AGTR construction. Instead, constructing an AGTR requires applying domain knowledge from a problem space to group data points with a high likelihood of sharing a reference label. Due to this domain knowledge requirement, each AGTR construction technique is specific to one kind of classification or clustering problem.

\subsubsection{Exploring the tightness or looseness of AGTR bounds}
\label{sec:trivial}
Some AGTR construction techniques will produce more useful evaluation metric bounds than others. Suppose a GTR $R$ constructed by simply assigning every data point to its own singleton cluster. We know that this method will always form a GTR with no errors because each singleton cluster must be the subset of some cluster in the ground truth reference clustering. However, when we use this GTR to compute the precision and recall of the predicted clustering, we obtain a precision lower bound of $\frac{c}{m}$ and a recall upper bound of $1$, where $c$ is the number of predicted clusters. These bounds are extremely loose and uninformative, and a GTR constructed in this manner is not useful for evaluation purposes.

We have found that the similarity in composition between an AGTR and the reference clustering strongly influences the tightness or looseness of evaluation metric bounds. Given a ground truth reference clustering $D$ and an AGTR $\hat{R}$, let $\delta$ be the minimum number of data points in $\hat{R}$ whose cluster membership must be changed in order to transform it into $D$. Using \autoref{thm:precision1err} and \autoref{thm:recall1err} we show that the difference between a metric bound (prior to accounting for $\hat{\epsilon}$) and the true value of that metric is no greater than $\frac{\delta}{m}$.

\setcounter{theorem}{3}
\setcounter{corollary}{1}
\begin{corollary} \label{thm:precisionerrbound}
$\lvert$\textit{Precision(C, D)} $-$ \textit{Precision(C, $\hat{R}$)}$\rvert$ $\leq \frac{\delta}{m}$
\end{corollary}

\setcounter{theorem}{4}
\setcounter{corollary}{1}
\begin{corollary} \label{thm:recallerrbound}
$\lvert$\textit{Recall(C, D)} $-$ \textit{Recall(C, $\hat{R}$)}$\rvert$ $\leq \frac{\delta}{m}$
\end{corollary}

\subsubsection{Properties of an ideal AGTR}
Because evaluation metric bounds can deviate by up to $\frac{\delta}{m}$ from the true metric values, AGTRs that have smaller values of $\frac{\delta}{m}$, \emph{i.e.}, ones that are as similar to the ground truth reference clustering as possible, are preferred. This allows us to identify the following properties that should be considered when designing an AGTR in order for it to produce tight evaluation metric bounds:

\emph{Low false positive rate.} An AGTR construction technique should group data points from different ground truth reference clusters as infrequently as possible. An increased rate of these false positives must be accounted for with a larger value of $\hat{\epsilon}$ to ensure that $\hat{\epsilon} \geq \epsilon$. This is undesirable, since a larger value of $\hat{\epsilon}$ results in looser evaluation metric bounds.

\emph{Acceptable false negative rate.} A method for constructing an AGTR should be effective at grouping together data points with the same ground truth reference label. An AGTR with too many ungrouped data points will have a large value of $\delta$, resulting in loose bounds, with precision tending to become very loose.

\emph{Scalable.} Datasets used for constructing an AGTR should be large enough to adequately represent the problem space. A technique for constructing an AGTR must have acceptable performance when applied to a large number of data points.

An AGTR with all three of these traits will produce tight evaluation metric bounds. However, in practice, we find that it is difficult to have a low false negative rate while at the same time avoiding false positives. We have frequently observed loose precision lower bounds due to false negatives during AGTR construction.

\section{Applications of Approximate Ground Truth Refinements} \label{sec:how_to_apply}

As long as an AGTR can be constructed from a dataset, that dataset can be used for partial evaluation of a clustering algorithm or multi-class classifier - even if the dataset does not have reference labels. This is notable because it allows larger, more representative datasets without reference labels to be used during the evaluation process. The ability to compute bounds on precision, recall, and accuracy without reference labels is valuable in its own right, but we have found two additional ways in which AGTRs can be used to great effect. First, the bounds from an AGTR can be used as a litmus test for detecting biased evaluation results produced using a substandard reference dataset. Second, AGTRs can be used for evaluating modifications to a clustering algorithm or multi-class clasifier.

\subsection{Testing Over-fit Evaluation Results}
\label{sec:testing-classifiers}

We have found that AGTRs can be used to detect misleading results produced from low quality reference datasets (\emph{e.g.} datasets that are small, under-representative, or have lower-confidence reference labels). The following five steps describe how to test for over-fit evaluation results:

\begin{enumerate}[label=\arabic{*}.]
\item Compute the precision, recall, and/or accuracy of a clustering algorithm or a multi-class classifier using a substandard reference dataset.

\item Obtain $C$ by applying the clustering algorithm or multi-class classifier to a large, diverse dataset.

\item Construct an AGTR $\hat{R}$ from the dataset in step 2. 

\item Select a value of $\hat{\epsilon}$ that is believed to be greater than the number of errors in $\hat{R}$.

\item Compute \textit{Precision(C, $\hat{R}$)} $-$ $\frac{\hat{\epsilon}}{m}$ and \textit{Recall(C, $\hat{R}$)} $+$ $\frac{\hat{\epsilon}}{m}$. Test whether all bounds hold for the evaluation results found during step 1. 
\end{enumerate}

Note that if testing a multi-class classifier, $C$ is obtained during step 2 by using the classifier to predict the class label for each data point in the dataset and then assigning all data points that share a predicted label to the same cluster. Subsection \ref{sec:IdealAGTR} describes the qualities an ideal AGTR should possess for best evaluation results. Refer to \autoref{sec:AGTRerror} for discussion about how to select an appropriate value of $\hat{\epsilon}$ during step 4. 

Because the dataset used to construct the AGTR during step 3 is larger and more diverse than the reference dataset from step 1, we assume that the AGTR dataset is a better exemplar of the overall problem space. Although the metric bounds found using the AGTR dataset do not necessarily hold for the reference dataset, it would still be irregular for any evaluation results from the reference dataset to violate them. Therefore, if \textit{Precision(C, $\hat{R}$)} $-$ $\frac{\hat{\epsilon}}{m}$ is greater than the precision computed using the reference dataset, or if \textit{Recall(C, $\hat{R}$)} + $\frac{\hat{\epsilon}}{m}$ is less than the recall or accuracy computed using the reference dataset, it is reasonable to conclude that the evaluation results found using the reference dataset are over-fit.

\subsection{Comparing Similar Models}
\label{sec:comparing-classifiers}

An important component of cluster and classifier evaluation is the ability to compare models against each other. Suppose we wish to determine which of two clustering algorithms has a higher precision, but we do not have access to a satisfactory reference dataset. We use the two models to predict clusterings $C_{1}$ and $C_{2}$ from an unlabeled dataset and we construct an AGTR $\hat{R}$ from the same dataset. The precision lower bounds of the two clustering algorithms are given by \textit{Precision($C_{1}$, $\hat{R}$)} $-$ $\frac{\hat{\epsilon}}{m}$ and \textit{Precision($C_{2}$, $\hat{R}$)} $-$ $\frac{\hat{\epsilon}}{m}$. Although we can apply \autoref{thm:precisionagtr} to show \textit{Precision($C_{1}$, $\hat{R}$)} $-$ $\frac{\hat{\epsilon}}{m}$ $\leq$ \textit{Precision($C_{1}$, D)} and \textit{Precision($C_{2}$, $\hat{R}$)} $-$ $\frac{\hat{\epsilon}}{m}$ $\leq$ \textit{Precision($C_{2}$, D)}, we cannot prove any relationship between \textit{Precision($C_{1}$, D)} and \textit{Precision($C_{2}$, D)}.

Unfortunately, evaluation metric bounds cannot be used to provably determine whether one clustering algorithm or multi-class classifier has a higher precision, recall, or accuracy than another. However, in specific cases higher evaluation metric bounds may indicate that one model has a higher performance than another. Three conditions must be met in order for this approach to be used. First, the clustering algorithms or classifiers being compared must be intrinsically similar, such as two different versions of the same classifier. Second, the problem space must have a large number of classes, such that it would be improbable for related data points to be grouped by random chance. Finally, one of the clustering algorithms or classifiers must be tested to ensure that changes in performance are strongly correlated to changes in evaluation metric bounds. The steps of the test are as follows:

\begin{enumerate}[label=\arabic{*}.]
\item Obtain $C$ by applying the clustering algorithm or multi-class classifier to a large, diverse dataset.

\item Construct an AGTR $\hat{R}$ from the dataset in step 1. 

\item Incrementally shuffle the cluster membership of each data point in $C$.

\item Compute \textit{Precision(C, $\hat{R}$)} $-$ $\frac{\hat{\epsilon}}{m}$ and \textit{Recall(C, $\hat{R}$)} $+$ $\frac{\hat{\epsilon}}{m}$ at regular intervals of the shuffle. 

\item Compute correlation between shuffle percentage and the evaluation metric bounds. Test that a strong negative correlation between the two exists.
\end{enumerate}

Step 3 is performed randomly sampling a data point in $C$ with no replacement, randomly selecting a cluster in $C$ weighted by the original distribution of cluster sizes, and then assigning the data point to that cluster. The process repeats $m$ times, where $m$ is the number of points in the dataset, after which the entire clustering has been shuffled. Because the dataset has high class diversity, the probability that a data point is randomly assigned to an incorrect cluster is far greater than the probability that it is randomly assigned the correct one. Therefore, shuffling $C$ with this strategy is very likely to produce predicted clusterings that are sequentially worse. Step 5 is performed by computing the Pearson correlation between the shuffle percentage and the evaluation metric bounds. If a very strong negative correlation exists between the two, we conclude that small modifications to a clustering algorithm or classifier that improve its performance will be reflected by a higher evaluation metric bounds and modifications that lower its performance will be reflected by lower bounds. We provide examples of using an AGTR to test evaluation results produced by low-quality datasets and to compare performance changes in similar models in subsections \ref{sec:avclass-test} and \ref{sec:comparing-avclass} respectively.

\section{Evaluating Malware Classifiers Using an AGTR}
\label{sec:ValidatingMalware}

A malware family is a collection of malicious files that are derived from a common source code. Classifying a malware sample into a known family provides insights about the behaviors it likely performs and can substantially assist triage and remediation efforts. Developing models that can classify or cluster malware by family is a vital research area \cite{Chakraborty}, but the task of obtaining family labels to train and/or evaluate these models is expensive and error prone, more so than most standard ML applications areas~\cite{Raff2020a}. This forces datasets to be either small and thus non-representative of the large and diverse malware ecosystem, or large, but noisily labeled or biased which again limits conclusions. Because of these issues, we believe that malware family classification is an ideal problem space for demonstrating the AGTR evaluation framework. In this section, we will quickly review the challenges in labeling malware and how it has caused a lack of quality datasets. Then, we will propose a method for constructing an AGTR from a dataset of malware samples.

\subsection{Malware Label and Dataset Challenges}
\label{sec:dataset-challenges}

Ideally a malware reference dataset would be constructed using \emph{manual labeling} to determine the malware family of each file. Manual analysis is not perfectly accurate, but the error rate is considered negligible enough that labels obtained via manual analysis are considered to have ground truth confidence~\cite{mohaisen2015}. A professional analyst can take a 10 hours or more to fully analyze a single file~\cite{mohaisen2013,10.1145/3290607.3313040}. This level of analysis is not always needed to determine the family of a malware sample, but it exemplifies the high human cost of manual labeling~\cite{perdisci}. We are not aware of any datasets in which the creators manually analyzed all malware samples in order to determine family labels. Instead, a few reference datasets use the family labels in \emph{open-source threat reports} published by reputable cybersecurity companies as a source of ground truth \cite{malgenome}. However, this limits both the size and contents of a dataset to the malware samples published in these reports. An approach that attempts to mitigate these scalability issues is \emph{cluster labeling}, in which the dataset is clustered and an exemplar from each cluster is manually labeled. This strategy is highly reliant on the precision of the algorithm used to cluster the reference dataset, which is often custom-made \cite{malsign,nappa}. Furthermore, the scalability of cluster labeling is still somewhat limited due to the requirement of manual analysis.

For these reasons, larger malware reference datasets tend to use \emph{antivirus labeling}, where an antivirus engine is used to label a corpus. This is the simplest method to implement at scale, but at the cost of label quality~\cite{zhu}. Antivirus labels are frequently incomplete, inconsistent, or incorrect \cite{botacin,mohaisen2014}. Antivirus signatures do not always contain family information~\cite{mohaisen2015}, and different antivirus engines disagree on the names of malware families \cite{avclass}. Labels from an antivirus engine can take almost a year to stabilize \cite{zhu}, creating a necessary lag time between data occurrence and label inference. These issues can be partially mitigated by using \emph{antivirus majority voting}, but aggregating the results of multiple antivirus engines produces unlabeled files due to lack of consensus. This then biases the final dataset to only the ``easy'' samples that are the least interesting and already known by current tools \cite{li}.

The aforementioned labeling issues obstruct creation of large malware datasets with high-confidence labels. The largest datasets used in malware classifier and clustering research range from a hundred thousand \cite{malsign,mohaisen2015} up to one million samples \cite{huang}, but are private corpora held by corporations that can afford the construction cost and do not want to give away a competitive advantage\footnote{There are also legal concerns for sharing benign applications, but our discussion is focused solely on malware.}. Since the data is private, the validation of the labeling can not be replicated or investigated, and in most cases, the number of families is not fully specified \cite{malsign,mohaisen2015}. 
The majority of publicly available datasets that have been used are less than 12,000 samples in size \cite{malheur,rieck,drebin,arp,nappa,kaggle}. Of these MalGenome is the only corpus with ground truth, but also the smallest with only 49 families and 1,260 files\cite{malgenome,zhou}. The small size and low diversity of these corpora make it difficult to form generalizable conclusions about the quality of a malware clustering algorithm or classifier. Additional statistics about notable malware  datasets are listed in Tables \ref{tab:privatedatasets} and \ref{tab:malware-datasets}.

\begin{table}[!t]
\centering
\caption{Notable Private Malware Reference Datasets}
\label{tab:privatedatasets}
\resizebox{\columnwidth}{!}{%
\begin{tabular}{@{}lrrlll@{}}
\toprule
Name & Samples & Families & Platform & Collection Period & Labeling Method\\ \midrule
    Malsign \cite{malsign} & 142,513 & Unknown & Windows & 2012 - 2014 & Cluster labeling\\ 
    MaLabel \cite{mohaisen2015} & 115,157 & $\ge$ 80 & Windows & Apr. 2015 or earlier & AV Majority Vote\\
    MtNet \cite{huang} & 1,300,000 & 98 & Windows & Jun. 2016 or earlier & Hybrid\\ \bottomrule
\end{tabular}
}
\end{table}

\begin{table}[!t]
\centering
\caption{Notable Public Malware Reference Datasets}
\label{tab:malware-datasets}
\resizebox{\columnwidth}{!}{%
\begin{tabular}{@{}lrrlll@{}}
\toprule
Name & Samples & Families & Platform & Collection Period & Labeling Method\\ \midrule
    Malheur \cite{malheur} & 3,133 & 24 & Windows & 2006 - 2009 & AV Majority Vote\\
    MalGenome \cite{malgenome} & 1,260 & 49 & Android & Aug. 2010 - Oct. 2011 & Threat Reports\\
    Drebin \cite{drebin} & 5,560 & 179 & Android & Aug. 2010 - Oct. 2012 & AV Majority Vote\\
    VX Heaven \cite{vxheaven} & 271,092 & 137 & Windows & 2012 or earlier & Single AV\\
    Malicia \cite{malicia} & 11,363 & 55 & Windows & Mar. 2012 - Mar. 2013 & Cluster Labeling\\
    AMD \cite{AMD} & 24,553 & 71 & Android & 2010 - 2016 & Cluster Labeling\\
    Kaggle \cite{kaggle} & 10,868 & 9 & Windows & Feb. 2015 or earlier & Susp. Single AV\\
    MalDozer \cite{maldozer} & 20,089 & 32 & Android & Mar. 2018 or earlier & Susp. Single AV\\
    EMBER2018 \cite{ember} & 485,000 & 3,226 & Windows & 2018 or earlier & AVClass \\
    \bottomrule
\end{tabular}
}
\end{table}

\label{sec:def_ref} 

During the course of our survey, we identified three traits of malware reference datasets that negatively impact evaluation:

 \emph{1) Reference labels without ground truth confidence:} Because producing ground truth reference labels for a large corpus of malware samples is infeasible, it is common practice to use malware reference labels without ground truth confidence \cite{perdisci}. Many prior malware classifiers have been evaluated using non-ground truth reference labels without the quality of those labels having been assessed \cite{avclass}, which may result in overly-optimistic or misleading results \cite{li}.

\emph{2) Insufficient size or diversity:} Due to the enormous number of families in existence, malware reference datasets must be both large and diverse in order to be representative of the malware ecosystem. Using a reference dataset with a small number or imbalanced distribution of families lowers the significance of evaluation results \cite{li}. 

\emph{3) Outdated malware samples:} The ecosystem of malware is constantly changing as categories of malware, malware families, and other tradecraft rise to prominence or fall out of favor. Evaluating a clustering algorithm or multi-class classifier using a dataset of outdated malware may produce results that do not translate to present day malware. All of the discussed datasets contain malware samples from 2015 and earlier, failing to represent the last half decade of malware development.

All of the datasets in our survey possess one or more of these undesirable characteristics, causing all evaluation results obtained using them to be dubious. In fact, due to the continuously evolving environment of malware development and the time required to manually label malware (as well as for antivirus labels to stabilize), it is impossible for a reference dataset to be both up-to-date and have ground truth reference labels \cite{kantchelian}. Therefore, the presence of at least one of these negative traits in a malware reference dataset is an unfortunate necessity.

The dearth of satisfactory malware reference datasets is primarily due to the difficulty of determining the the family name for each malware sample. However, identification of related malware samples is a well-researched topic and it can be performed efficiently as well as with high fidelity. In short, \textit{although it is difficult to label malware samples, it is easy to group them}. By identifying a subset of the relationships in a large corpus of malware samples, information about the performance of a malware classifier or clustering algorithm can be obtained without knowing any reference labels. It is this concept which forms the foundation of the AGTR  framework.

\subsection{Constructing a Malware Dataset AGTR}
\label{sec:pehash}

In this section, we discuss a method for constructing an AGTR from a dataset of malware samples. Because our method is automatic and scalable, the resulting AGTR can be orders of magnitude larger than a ground truth malware reference dataset and can include modern malware samples. Our method for constructing an AGTR from a malware dataset is based on peHash, a metadata hash for files in the Portable Executable (PE) format. Files in the PE format are executable files that can run on the Windows operating system, such as .exe, .dll, and .sys files. peHash was designed for identifying polymorphic malware samples within the same family as well as nearly identical malware samples. The hash digest is computed using metadata from the PE file header, PE optional header, and each PE section header \cite{pehash}. Two malware samples with identical values for all of the chosen metadata features have identical peHash digests. Due to the number of metadata features used in the hash and the large range of possible values that these features can have, the odds that two unrelated malware samples share a peHash digest is minuscule.

Our proposed method for constructing an AGTR from a Windows malware dataset requires computing the peHash digest of each malware sample. Then, all malware samples that share a peHash digest are assigned to the same cluster. If the peHash of a malware sample cannot be computed, such as due to malformed PE headers, it is assigned to a singleton cluster. Using a hash table to tabulate clusters allows an AGTR to be built very efficiently, requiring only $O(m)$ memory usage and $O(m)$ run time complexity, where $m$ is the number of malware samples in the dataset.

\citet{pehash} evaluated the false positive rate of peHash using 184,538 malware samples from the mwcollect Alliance dataset and 90,105 malware samples in a dataset provided by Arbor Networks. All malware samples were labeled using the ClamAV antivirus engine \cite{clamav}. The peHash of each malware sample in both datasets was calculated, resulting in 10,937 clusters for the mwcollect Alliance dataset and 21,343 clusters for the Arbor Networks dataset. Of these clusters, 282 and 322 had conflicting antivirus labels respectively. However, manual analysis showed that none of the clusters with conflicting antivirus labels contained unrelated malware samples. The evaluation method Wicherski used
does not rule out the possibility of false positives. However, it is evident that the false positive rate of peHash is extremely low. 
Based on Wicherski's evaluation and our own additional assessment, we suggest choosing an $\hat{\epsilon}$ of approximately one percent the total dataset size when using a peHash AGTR. We believe that this value should far exceed the true number of errors $\epsilon$ in the AGTR.

A major consideration in the selection of peHash as our proposed AGTR construction method is its prevalent industry use. peHash is widely regarded to have an extremely low false positive rate. Furthermore, due to the adversarial nature of the malware ecosystem, \citet{pehash} has already analyzed peHash's vulnerabilities and its widespread usage in industry means practitioners are aware of the real-world occurrence of attacks against it. These factors allow us to be very confident in our assessment of peHash's error rate. It was for these reasons that we elected to use peHash rather than design a custom AGTR construction technique. Developing new
methods for constructing AGTRs is a target of future work that may yield tighter evaluation metric bounds.

\section{Case Study: Applying the AGTR Evaluation Framework to AVClass} \label{sec:avclass_agtr_eval}

At this point, we have established the AGTR evaluation framework, discussed how malware classifier evaluation can benefit from it, and introduced a method for constructing an AGTR from a dataset of Windows malware using peHash. We will now apply the AGTR evaluation framework with the approaches described in subsections \ref{sec:testing-classifiers} and \ref{sec:comparing-classifiers} to the malware labeling tool AVclass \cite{avclass}. We hope that this section provides a template for utilizing the AGTR evaluation framework in other clustering and classification problem spaces.

When provided an antivirus scan report for a malware sample, AVClass attempts to aggregate the many antivirus signatures in the report into a single family label. AVClass is open source, simple to use, and does not require the malware sample to obtain a label, making it a popular choice as a malware classifier since its release in 2016. We provide new evidence of over-fitting in the original AVClass evaluation results due to the use of poor reference data. We also demonstrate the ability to compare modified versions of classifiers using an AGTR by making minor modifications to AVClass and assessing their benefits or drawbacks. Evaluating such nuanced modifications was not previously tenable due to the lack of large reference datasets. The ability to compare the impact of model adjustments immediately, that are otherwise hard to detect, is of significant value in this domain, as production changes usually require months to obtain customer feedback or through ``phantom'' deployments (i.e., a new model is deployed alongside a previous model, but the new results are recorded for evaluation and comparison). 

\subsection{Testing AVClass Results Using an AGTR}
\label{sec:avclass-test}

Sebastian \emph{et al.} \cite{avclass} evaluated AVClass using five malware reference datasets. Because security vendors frequently refer to malware families by different names, the family names used by AVClass do not match those used by the reference datasets. Therefore, although AVClass is a classifier, Sebastian \emph{et al.} could not compute its accuracy and chose to use precision and recall instead.

\begin{table}[h]
\centering
\caption{AVClass Precision and Recall \cite{avclass}}
\label{tab:avclassmetrics}
\begin{tabular}{lcc}
\toprule
Dataset Name & Precision & Recall \\ \midrule
    Drebin & 0.954 & 0.884 \\
    Malicia & 0.949 & 0.680 \\
    Malsign & 0.904 & 0.907 \\
    MalGenome* & 0.879 & 0.933 \\
    Malheur & 0.904 & 0.983 \\ \bottomrule
\end{tabular}
\end{table}

Precision and recall scores for the default version of AVClass are shown in Table \ref{tab:avclassmetrics}. The row entitled MalGenome* is a modified version of the MalGenome dataset where labels for six variants of the DroidKungFu family are corrected. We call attention to the high variation in evaluation results - the precision of AVClass ranges from 0.879 to 0.954 and its recall ranges from 0.680 to 0.983. It is clear that due to these inconsistencies, the evaluation results for AVClass are already suspect. To confirm this, we test the evaluation results of AVClass using the method described in \autoref{sec:testing-classifiers}.

To construct an AGTR we use a portion of the VirusShare dataset \cite{virusshare}. The full VirusShare corpus contains 40,620,999  unlabeled malware samples dated between June 2012 and the time of writing. The VirusShare dataset is broken into chunks and new chunks are added to the dataset regularly. We were provided antivirus scan reports for chunks 0-7, which consist of 1,048,567 malware samples \cite{seymour}. These scans were collected between December 2015 and May 2016 by querying the VirusTotal API \cite{virustotal}. We ran AVClass under default settings to obtain predicted family labels from each scan report. We produced a predicted clustering $C$ by assigning all malware samples with the same AVClass label to the same cluster. Malware samples for which no label could be determined were assigned to singleton clusters. Next, we created a peHash AGTR $\hat{R}$ from VirusShare chunks 0-7. Following our recommendation in \autoref{sec:pehash}, we choose $\hat{e} =$ 10,000 for the AGTR, which allows for an error rate of up to approximately one percent during the AGTR construction process - an error rate which should comfortably account for the very low number of errors we expect to be present in the peHash AGTR.

Using the peHash AGTR, we found that \textit{Precision(C, $\hat{R}$)} $- \frac{\hat{\epsilon}}{m} = 0.229$ and \textit{Recall(C, $\hat{R}$)} $+ \frac{\hat{\epsilon}}{m} = 0.895$. As a result of our analysis, we conclude that AVClass has an accuracy no greater than 0.895. The precision lower bound of 0.229 seems to be very loose considering that the smallest precision in Table \ref{tab:avclassmetrics} is 0.879. We attribute this to the moderate false negative rate of peHash; an AGTR construction technique with a lower false negative rate should yield a tighter bound. The similarity between the recall upper bound and the reported recall results shows that, although our peHash AGTR could be improved, the bounds are non-trivial. Designing improved methods for constructing AGTRs from malware datasets is an issue for future work. The Malsign, MalGenome, and Malheur datasets in Table \ref{tab:avclassmetrics} all have recall values exceeding the upper bound found using the peHash AGTR; the values for MalGenome and Malheur especially so. Although the exact number of families in VirusShare chunks 0-7 is unknown, we estimate that it contains samples from thousands of distinct malware families. Because VirusShare chunks 0-7 is much larger (exceeding 1 million malware samples) and more diverse than the Malsign, MalGenome, and Malheur datasets, we believe that evaluation results produced using those datasets are over-fit, likely due to the labeling difficulties we discussed in \autoref{sec:def_ref}.

\subsection{Comparing Modified Versions of AVClass}
\label{sec:comparing-avclass}

In this section, we show that a peHash AGTR can be used to determine whether modifications to AVClass make a positive or negative impact on performance. In order to compare clustering algorithms or classifiers using an AGTR, they must meet the three conditions listed in \autoref{sec:comparing-classifiers}. Because we are comparing AVClass to slightly modified versions of itself, the classifiers are similar enough that the first condition is met. Since we estimate that VirusShare chunks 0-7 contain thousands of malware families, it would be extremely rare to guess the correct family by random chance, which satisfies the second condition. For the third condition, we must determine if changes in classifier performance are strongly correlated with the evaluation metric bounds. To perform this check we use the same predicted clustering $C$ and AGTR $\hat{R}$ from \autoref{sec:avclass-test}. Next, we incrementally shuffle $C$ and compute the precision and recall bounds each time that an additional one percent of the data points have been shuffled.

\begin{figure}[h]
    \centering
    \adjustbox{max width=\columnwidth}{%
        \includegraphics{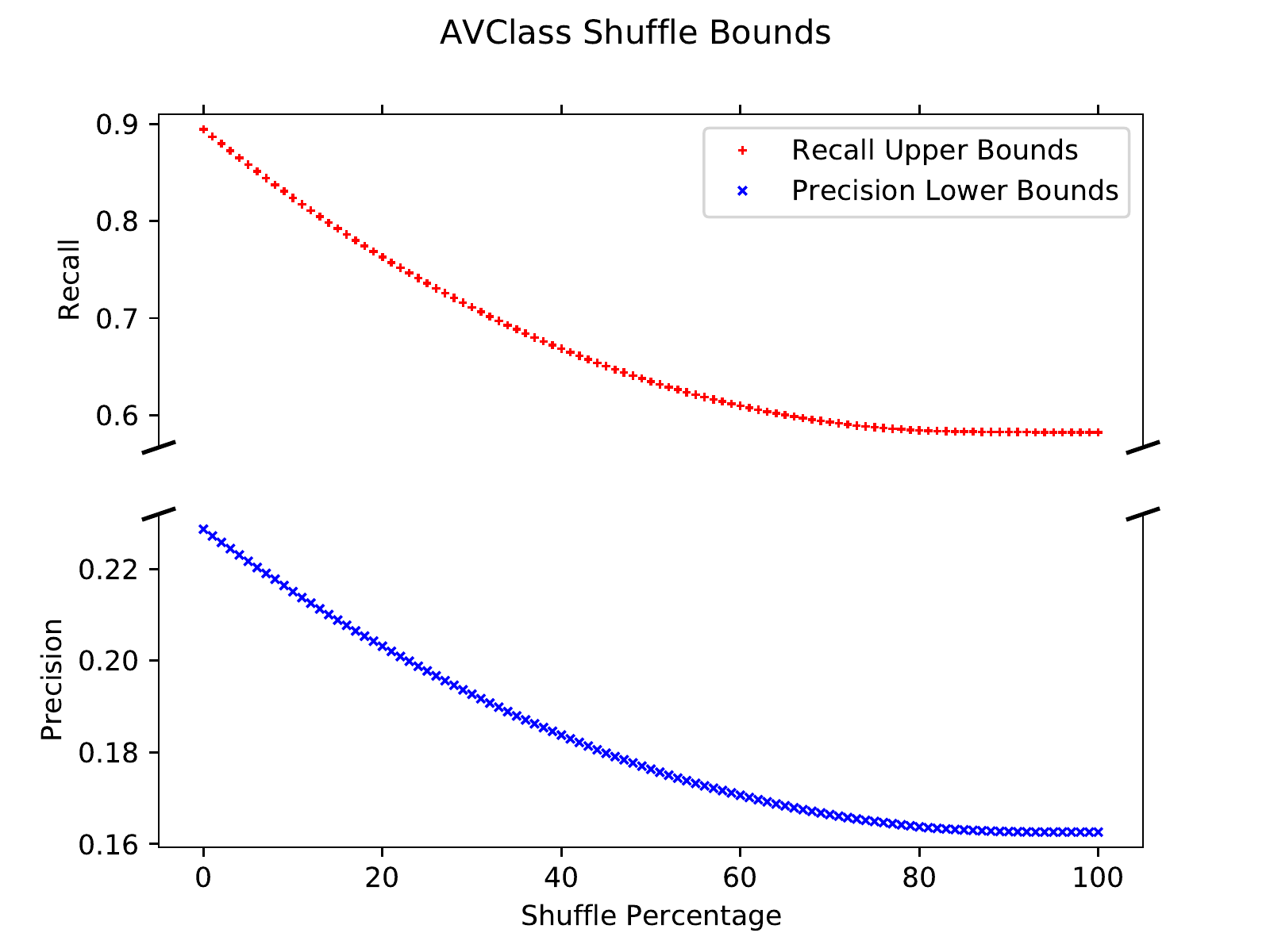}
    }
    \caption{Precision and recall bounds of AVClass with respect to shuffle percentage. The x-axis of each figure shows the percentage of data points whose cluster membership has been shuffled. The y-axis of each figure shows the value of the metric bound. As the shuffle percentage increases, our bounds adjust monotonically and at a near linear rate, slowing only after 80\% corruption.}
    \label{fig:Comparing}
\end{figure}

\autoref{fig:Comparing} shows how the bounds change as the data points are shuffled. It is evident that both bounds worsen predictably and monotonically as the data points are shuffled. 
The precison and recall bounds have a correlation of -0.956 and -0.940 respectively with the percentage of labels shuffled, each with a p-value $\leq 10^{-47}$.
Since a higher shuffle percentage indicates a worse clustering, there is likely a strong connection between the bounds and the true metric values. With all three required conditions having been met, we conclude that it is valid to compare modified versions of AVClass using the AGTR evaluation framework. Next, we compare modified versions of AVClass to the original tool. The purpose of this exercise is to demonstrate that an AGTR can be used to quantify the relative benefits and trade-offs of each of these changes to AVClass in the absence of reference data.

\subsubsection{Comparing Alias Resolution Methods in AVClass}

As we mentioned in Section \ref{sec:classifier-metrics}, it is common for different antivirus engines to refer to the same family of malware by different names. We call two names for the same malware family \emph{aliases} of each other. One of the steps that AVClass performs while aggregating antivirus signatures is resolution of family aliases \cite{avclass}. If aliases are not resolved properly, AVClass could produce erroneous labels. By default, AVClass uses a manually generated list of known aliases. AVClass also has a setting for generating a family alias map based on families that have a high co-occurrence percentage within a corpus of antivirus scan results. Generation of the family alias map is controlled by the parameters $n_{alias}$, which is the minimum number of malware samples two tokens must appear in together, and $T_{alias}$, which is the minimum co-occurrence percentage.

To investigate how alias replacement affects label quality, we provided AVClass with three different family alias maps and use it to label VirusShare chunks 0-7. The first map is the manually-produced one packaged within AVClass which is used by default. We generated the second map using the recommended parameter values $n_{alias} =$ 20 and $T_{alias} =$ 0.94, which were chosen empirically by Sebastian \emph{et al} \cite{avclass}. The third map was generated using the stricter parameter values $n_{alias} =$ 100 and $T_{alias} =$ 0.98 listed in the AVClass documentation.

\begin{table}[h]
\centering
\renewcommand{\arraystretch}{1.4}
\caption{Alias Resolution Bounds}
\label{tab:avclass_alias_eval}
\adjustbox{max width=\columnwidth}{%
\begin{tabular}{@{}lccc@{}}
\toprule
Alias Preparation    & Manual       & Recommended       & Strict \\ \midrule
Precision Lower Bound & 0.229 & 0.230 & 0.233 \\
Recall Upper Bound & 0.895 & 0.897 & 0.894 \\ \bottomrule
\end{tabular}
}
\end{table}

Table \ref{tab:avclass_alias_eval} shows the precision lower bound and recall upper bound for AVClass using the three family alias mappings. Generating a family alias map using the recommended parameters yields both higher precision and recall bounds than the default one. Generating a map using the stricter parameters results in the highest precision bound, but the lowest recall bound. An AVClass user can reasonably conclude that generating an alias mapping using the recommended settings is the best of the three options if high recall is desirable and that the strict settings are best if higher precision is required.

\subsubsection{Adding a Threshold to AVClass' Plurality Voting}

For the next modification, we added a plurality threshold to AVClass. By default, AVClass determines the label of a malware sample by selecting the plurality family proposed by the antivirus engines in a scan report \cite{avclass}. Rather than using simple plurality voting to determine the label, we modified AVClass to require that the number of votes for the plurality family exceeds the number of votes for any other family by a given threshold. For example, if the plurality threshold is two, the plurality family must receive at least two more votes than any other family. If no family meets this condition, AVClass outputs no label for that sample and it is assigned to a singleton cluster. We used the modified version of AVClass to label VirusShare chunks 0-7 with plurality thresholds between 0 and 5.

\begin{table}[h]
\caption{Plurality Threshold Bounds}
\label{tab:avclass_plurality_threshold_eval}
\centering
\adjustbox{max width=\columnwidth}{%
\begin{tabular}{@{}ccc@{}}
\toprule
Threshold & Precision Lower Bound & Recall Upper Bound \\ \midrule
0 & 0.229     & 0.895  \\
1 & 0.276     & 0.881  \\
2 & 0.332     & 0.860  \\
3 & 0.442     & 0.829  \\
4 & 0.511     & 0.803  \\
5 & 0.565     & 0.780  \\ \bottomrule
\end{tabular}
}
\end{table}

\autoref{tab:avclass_plurality_threshold_eval} displays the precision and recall bounds of AVClass using different plurality thresholds. Note that a plurality threshold of zero is equivalent to the default version of AVClass. As the plurality threshold is raised, the precision lower bound significantly increases, indicating that higher thresholds reduce the number of false positives. However, raising the plurality threshold creates a trade-off, as it causes the recall (and hence accuracy) upper bounds to decrease to a lesser degree. This decrease in recall is largely due to the growing number of unlabeled malware samples contributing to the false negative rate. Thresholds above three may be useful for classifiers that require a very high precision. A threshold of one or two may offer a higher precision than the default version of AVClass without sacrificing a significant amount of recall. Since different applications of malware classification may require either a low false positive rate or a low false negative rate, our findings indicate how designers of malware classifiers can adopt a suitable voting strategy.

\subsubsection{Removing Heuristic Antivirus Signatures in AVClass Voting}

When normalizing an antivirus signature, AVClass treats each token within the signature independently. However, we believe that incorporating contextual information from each token could improve AVClass' labeling decisions. A simple example of this is using context from tokens that indicate that the antivirus signature is a ``heuristic.'' We believe that heuristic signatures are more likely to include inaccurate family information. To test this, we have identified eight tokens that indicate that an antivirus signature is a heuristic: $\mathit{gen, generic, heur, heuristic, eldorado, variant, behaveslike}$, and $\mathit{lookslike}$. We modified AVClass to ignore any antivirus signatures which contain one or more of these tokens during voting.

\begin{table}[h]
\centering
\caption{Heuristic Signature Removal Bounds}
\label{tab:heuristic_eval}
\adjustbox{max width=\columnwidth}{%
\begin{tabular}{@{}lcc@{}}
\toprule
 & Default       & Heuristic Removal       \\ \midrule
Precision Lower Bound & 0.229 & 0.250  \\
Recall Upper Bound & 0.895 & 0.889 \\
\bottomrule
\end{tabular}
}
\end{table}

\autoref{tab:heuristic_eval} shows the evaluation metric bounds for the default version of AVClass and the modified version of AVClass where heuristic antivirus signatures are not counted towards the plurality vote. Simply ignoring common heuristic antivirus signatures substantially raises the precision bound of AVClass from 0.229 to 0.250. This comes at the cost of a minor increase in false negatives, as indicated by the slight drop in the recall bound. This confirms our suspicions that heuristic antivirus signatures often contain inaccurate family information. A more sophisticated method for handling heuristic    signatures could offer even further improvements to AVClass.

\section{Discussion and Conclusion}
\label{sec:conclusion}

We conclude by providing further exposition as to how we envision AGTRs being used, as well as their constraints. To be clear, our proposed AGTR framework does not alleviate the need for labeled data nor does it avoid all possible biases. Indeed, by constructing a mechanism that produces AGTRs, assuming it does not produce a complete graph, there must necessarily be some form of bias that exists via the selection of the AGTR production itself. Not only is labeled data needed to produce the initial model being evaluated, but labeled data is necessary to produce the estimate of the error rate $\hat{\epsilon}$ of the AGTR construction. These are important nuances to keep in mind during usage. 

Our belief is that there are many domains where constructing an AGTR with confidently small $\hat{\epsilon}$ is possible. The  case study which we perform in  \autoref{sec:ValidatingMalware} describes the utility AGTRs in the field of malware analysis, as it is the area of our expertise. We see potential application areas in chemistry, physics, optics, and other fields where obtaining experimental ground truth results is highly expensive and time consuming, but sophisticated simulations may be able to automatically group different scenarios together with high fidelity. That is, the simulation's result may not be an exactly correct prediction, but scenarios are congregated into groups in which all elements have the same fundamental nature with high probability. This is extrapolation on our part, since we do not have the requisite background to delve deeply into these other areas, but it captures our primary hypothesis: many domains have the requisite domain knowledge that could produce viable AGTRs.

Once an AGTR is produced, it provides a valuable means of improving bench-marking, confidence in results and testing of nuanced changes. Designing good benchmarks is of critical importance and an AGTR can be an effective litmus test for any benchmark creation. Our framework allows one to cheaply use a much larger set of unlabeled data to determine if the benchmark itself has become the subject of over-fitting, by performing the test described in \autoref{sec:testing-classifiers}. We later provide an example of this with the AVClass tool in \autoref{sec:avclass-test}, where our bounds indicate AVClass's evaluation is likely over-fit because their empirical recall is higher than the recall upper bound we obtain using an AGTR. This does not necessarily mean the results are actually over-fit, because the datasets evaluated are not the same (in AVClass' cases, not available to us). However, this finding does provide a strong indicator that something is amiss and it signals a need to more thoroughly and rigorously investigate the AVClass evaluation results.

The utility of AGTRs to perform benchmarking is simultaneously macro and micro in nature. The prior paragraph has described the macro nature, in that we can perform litmus testing against a created benchmark to detect over-fitting against a larger unlabeled corpus. The next level down that one would naturally wish to pursue is comparing two different models using an AGTR, but we must caution the reader prior to doing such a comparison. When two models are being evaluated and are of different mechanical natures, they will naturally correlate to different degrees with the biases of the AGTR construction itself. This means one method could appear to have better predicted bounds by our AGTR by only being successful on the data for which the AGTR itself is able to operate, while potentially being errant on all other data and would not be realized. Any other more reasonable model, that simply lacks the bias to the AGTR in use, would then appear worse when it may in fact be superior. It is for this reason that we describe the strict conditions that must be met prior to comparing two models in \autoref{sec:comparing-classifiers}. The models being compared must be very fundamentally similar, such as changing small parameters of a single approach. This is reasonable because the underlying mechanism of prediction remains the same and allows us to better quantify the impacts of minute changes that would be difficult to estimate on small static corpora. This assumes that the parameters themselves do not have an out-sized impact on the fundamental nature, which we believe is reasonable, but should be made explicit. 

This highlights another limitation of our approach, which is the number of epistemic uncertainties involved in its use. It will depend on the beliefs and confidence of the domain experts that the models being compared are ``similar enough'' that using an AGTR to compare changes is valid and that the error rate bound $\hat{\epsilon}$ is sufficient. While more concrete answers to these are desirable, we recognize them as important items of future work. As we have developed the AGTR approach thus far, it provides a powerful means of better leveraging unlabelled data in the evaluation of a ML model, which we argue is additive to the tools in use today.

The AGTR evaluation framework allows bounds on precision, recall, and accuracy to be computed without labeled reference data, which becomes all the more important as datasets become too large to manually validate. In addition, we have designed a litmus test that uses AGTRs to identify over-fit evaluation results produced by low-quality reference datasets. We show that AGTRs can be used to evaluate the impact of changes made to a model. We identify malware family classification as a field where the AGTR validation framework provides value, provide an implementation for constructing an AGTR using peHash, and apply the AGTR evaluation framework to the AVClass malware labeler. There is no shortage of other problem spaces with inadequate reference datasets and we believe that this work can be used to improve the evaluation process for them as well.

\bibliographystyle{IEEEtranN}
\bibliography{sample-base}

\begin{appendices}

\section{Proof Details}
\label{sec:proofs}

\vspace*{0.25cm}

\setcounter{theorem}{0}
\setcounter{corollary}{0}

\begin{theorem}
\textit{Precision(C, R)} $\leq$ \textit{Precision(C, D)}
\end{theorem}
\begin{proof}[Proof of Theorem \ref{thm:precisiongtr}] \label{prf:precisiongtr}
Suppose some $i$ s.t. $1 \leq i \leq c$. Since $R$ is a refinement of $D$, $\exists D_{j} \in D$ s.t. $R_{f'(i)} \subseteq D_{j}$. Because $R_{f'(i)} \subseteq D_{j}$, it must be that $|C_{i} \cap R_{f'(i)}| \leq |C_{i} \cap D_{j}|$. By definition, $f(i) =$ arg$\max\limits_{j}|C_{i} \cap D_{j}|$. Therefore, $|C_{i} \cap R_{f'(i)}| \leq |C_{i} \cap D_{j}| \leq |C_{i} \cap D_{f(i)}|$. We can simply sum over this inequality to obtain:
$$\frac{1}{m} \sum\limits_{i=1}^{c} |C_{i} \cap R_{f'(i)}| \leq \frac{1}{m} \sum\limits_{i=1}^{c} |C_{i} \cap D_{f(i)}|$$ By Definition \ref{def:precision}, $Precision(C, R) \leq Precision(C, D)$.
\end{proof}

\vspace*{0.25cm}

\begin{theorem}
\textit{Recall(C, R)} $\geq$ \textit{Recall(C, D)}
\end{theorem}
\begin{proof}[Proof of Theorem \ref{thm:recallgtr}] \label{prf:recallgtr}
Suppose some $j$ s.t. $1 \leq j \leq d$. Because $R$ is a refinement of $D$, $\exists Q_{j} = \{Q_{j\ell}\}_{1 \leq \ell \leq q_{j}}$ s.t. $Q_{j\ell} \in R$ and $D_{j} = \bigcup\limits_{\ell=1}^{q_{j}}Q_{j\ell}$. We can say that $|D_{j}| = \sum\limits_{\ell=1}^{q_{j}}|Q_{j\ell}|$ and furthermore that $|C_{g(j)} \cap D_{j}| = \sum\limits_{\ell=1}^{q_{j}}|C_{g(j)} \cap Q_{j\ell}|$. We know that $\forall Q_{j\ell} \in Q_{j}$, $\exists R_{k} \in R$ s.t. $Q_{j\ell} = Q_{k}$. By definition $g'(k) =$ arg$\max\limits_{i}|C_{i} \cap R_{k}|$, and $R_{k} = Q_{j\ell}$, so $|C_{g'(k)} \cap R_{k}| \geq |C_{g(j)} \cap Q_{j\ell}|$. Since by Property 1 the sets $R_{k}$ are in bijection with the sets $Q_{j\ell}$, we can sum over this inequality to obtain: 
$$\frac{1}{m} \sum\limits_{k=1}^{r} |C_{g'(k)} \cap R_{k}| \geq \frac{1}{m} \sum\limits_{j=1}^{d} \sum\limits_{\ell=1}^{q_{j}} |C_{g(j)} \cap Q_{j\ell}|$$ $$\frac{1}{m} \sum\limits_{j=1}^{d} \sum\limits_{\ell=1}^{q_{j}} |C_{g(j)} \cap Q_{j\ell}| = \frac{1}{m} \sum\limits_{j=1}^{d} |C_{g(j)} \cap D_{j}|$$
By Definition \ref{def:recall}, \textit{Recall(C, R)} $\geq$ \textit{Recall(C, D)}.
\end{proof}

\vspace*{0.25cm}

\begin{corollary}
\textit{Recall(C, R)} $\geq$ \textit{Accuracy(C, D)}
\end{corollary}
\begin{proof} [Proof of Corollary \ref{cor:accuracygtr}] \label{prf:accuracygtr}
According to \autoref{thm:recallgtr}, \textit{Recall(C, R)} $\geq$ \textit{Recall(C, D)}. It must be the case that \textit{Recall(C, D)} where $g(j) = $ arg$\max\limits_{i}|C_{i} \cap D_{j}|$ $\geq$ \textit{Recall(C, D)} where $g$ is the identity function. Using Definition \ref{def:accuracy}, \textit{Recall(C, D)} = \textit{Accuracy(C, D)} when $g$ is the identity function. Therefore, \textit{Recall(C, R)} $\geq$ \textit{Accuracy(C, D)}.
\end{proof}

\vspace*{0.25cm}

\begin{theorem}
$\lvert$\textit{Precision(C, S)} $-$ \textit{Precision(C, $\hat{S}$)}$\rvert$ $\leq \frac{1}{m}$
\end{theorem}
\begin{proof}[Proof of Theorem \ref{thm:precision1err}] \label{prf:precision1err}
Let $S = \{S_{t}\}_{1\leq t \leq s}$ be an arbitrary partition of $M$. Let the function $f : \{1...c\} \mapsto \{1...s\}$ be defined as $f(i) = $ arg$\max\limits_{t}|C_{i} \cap S_{t}|$. Suppose some $M_{n} \in M$, some $S_{x} \in S$ s.t. $M_{n} \in S_{x}$, and some $S_{y}$ s.t. $S_{y} \in S$ or $S_{y} = \{\text{\O}\}$. Furthermore, suppose some $C_{a}, C_{b} \in C$ s.t. $M_{n} \in C_{a}$ and $b = f(y)$. Let $\hat{S} = \{\hat{S}_{\hat{t}}\}_{1\leq \hat{t} \leq \hat{s}}$ be a clustering identical to $S$ except for one cluster label change, which is given by $\hat{S}_{x}$ = $S_{x} - \{M_{n}\}$ and $\hat{S}_{y}$ = $S_{y} \cup \{M_{n}\}$. Let the function $\hat{f} : \{1...c\} \mapsto \{1...\hat{s}\}$ be defined as $\hat{f}(i) = $ arg$\max\limits_{\hat{u}}|C_{i} \cap \hat{S}_{\hat{u}}|$. At minimum $|C_{a} \cap \hat{S}_{\hat{f}(a)}| = |C_{a} \cap S_{f(a)}| - 1$ and at maximum $|C_{a} \cap \hat{S}_{\hat{f}(a)}| = |C_{a} \cap S_{f(a)}|$. Likewise, at minimum $|C_{b} \cap \hat{S}_{\hat{f}(b)}| = |C_{b} \cap S_{f(b)}|$ and at maximum $|C_{b} \cap \hat{S}_{\hat{f}(b)}| = |C_{b} \cap S_{f(b)}| + 1$. Finally, because all elements in $M$ are unique and $C$ partitions $M$, $\forall C_{i} \in C$ s.t. $i \neq a$ and $i \neq b$, $|C_{i} \cap \hat{S}_{\hat{f}(i)}| = |C_{i} \cap S_{f(i)}|$. Therefore, $\sum\limits_{t=1}^{s} |C_{g(t)} \cap S_{t}| - 1 \leq \sum\limits_{\hat{t}=1}^{\hat{s}} |C_{\hat{g}(\hat{t})} \cap \hat{S}_{\hat{t}}| \leq \sum\limits_{t=1}^{s} |C_{g(t)} \cap S_{t}| + 1$. By Definition \ref{def:precision}, \textit{Precision(C, S)} $-$ $\frac{1}{m} \leq$ \textit{Precision(C, $\hat{S}$)} $\leq$ \textit{Precision(C, S)} $+$ $\frac{1}{m}$. We write this as $\lvert$\textit{Precision(C, S)} $-$ \textit{Precision(C, $\hat{S}$)}$\rvert$ $\leq \frac{1}{m}$. 
\end{proof}

\vspace*{0.25cm}

\begin{corollary}
$\lvert$\textit{Precision(C, R)} $-$ \textit{Precision(C, $\hat{R}$)}$\rvert$ $\leq \frac{\epsilon}{m}$
\end{corollary}
\begin{proof}[Proof of Corollary \ref{thm:precisionerr}] \label{prf:precisionerr}
By Theorem \ref{thm:precision1err}, $\lvert$\textit{Precision(C, S)} $-$ \textit{Precision(C, $\hat{S}$)}$\rvert$ $\leq \frac{1}{m}$ for some arbitrary clustering $S$ and a second clustering $\hat{S}$ that is equivalent to $S$ but with a single data point belonging to a different cluster. Given a GTR $R$ and a corresponding AGTR $\hat{R}$, we can sequentially change the cluster membership of $\epsilon$ data points in $R$ to obtain $\hat{R}$. At each step the precision value can change by at most $\pm\frac{1}{m}$. Therefore, $\lvert$\textit{Precision(C, R)} $-$ \textit{Precision(C, $\hat{R}$)}$\rvert$ $\leq \frac{\epsilon}{m}$.
\end{proof}

\vspace*{0.25cm}

\begin{corollary} $\lvert$\textit{Precision(C, D)} $-$ \textit{Precision(C, $\hat{R}$)}$\rvert$ $\leq \frac{\delta}{m}$
\end{corollary}
\begin{proof}[Proof of Corollary \ref{thm:precisionerrbound}] \label{prf:precisionerrbound}
By Theorem \ref{thm:precision1err}, $\lvert$\textit{Precision(C, S)} $-$ \textit{Precision(C, $\hat{S}$)}$\rvert$ $\leq \frac{1}{m}$ for some arbitrary clustering $S$ and a second clustering $\hat{S}$ that is equivalent to $S$ but with a single data point belonging to a different cluster. Given a ground truth reference clustering $D$ and a corresponding AGTR $\hat{R}$, we can sequentially change the cluster membership of $\delta$ data points in $D$ to obtain $\hat{R}$. At each step the precision value can change by at most $\pm\frac{1}{m}$. Therefore, $\lvert$\textit{Precision(C, D)} $-$ \textit{Precision(C, $\hat{R}$)}$\rvert$ $\leq \frac{\delta}{m}$.
\end{proof}

\vspace*{0.25cm}

\begin{theorem}
$\lvert$\textit{Recall(C, S)} $-$ \textit{Recall(C, $\hat{S}$)}$\rvert$ $\leq \frac{1}{m}$
\end{theorem}
\begin{proof}[Proof of Theorem \ref{thm:recall1err}] \label{prf:recall1err}
Let $S = \{S_{t}\}_{1\leq t \leq s}$ be an arbitrary partition of $M$. Let the label function $g : \{1...s\} \mapsto \{1...c\}$ be defined as $g(t) = $ arg$\max\limits_{i}|C_{i} \cap S_{t}|$. Suppose some $M_{n} \in M$, some $S_{x} \in S$ s.t. $M_{n} \in S_{x}$, and some $S_{y}$ s.t. $S_{y} \in S$ or $S_{y} = \{\text{\O}\}$. Let $\hat{S} = \{\hat{S}_{\hat{t}}\}_{1\leq \hat{t} \leq \hat{s}}$ be a clustering identical to $S$ except for one cluster label change, which is given by $\hat{S}_{x}$ = $S_{x} - \{M_{n}\}$ and $\hat{S}_{y}$ = $S_{y} \cup \{M_{n}\}$. Let the function $\hat{g} : \{1...\hat{s}\} \mapsto \{1...c\}$ be defined as $\hat{g}(\hat{t}) = $ arg$\max\limits_{i}|C_{i} \cap \hat{S}_{\hat{t}}|$. At minimum $|C_{\hat{g}(x)} \cap \hat{S}_{x}| = |C_{g(x)} \cap S_{x}| - 1$ and at maximum $|C_{\hat{g}(x)} \cap \hat{S}_{x}| = |C_{g(x)} \cap S_{x}|$. Similarly, at minimum $|C_{\hat{g}(y)} \cap \hat{S}_{y}| = |C_{g(y)} \cap S_{y}|$ and at maximum $|C_{\hat{g}(y)} \cap \hat{S}_{y}| = |C_{g(y)} \cap S_{y}| + 1$. Because each other cluster in $\hat{S}$ is identical to some cluster in $S$, $\sum\limits_{t=1}^{s} |C_{g(t)} \cap S_{t}| - 1 \leq \sum\limits_{\hat{t}=1}^{\hat{s}} |C_{\hat{g}(\hat{t})} \cap \hat{S}_{\hat{t}}| \leq \sum\limits_{t=1}^{s} |C_{g(t)} \cap S_{t}| + 1$. Using Definition \ref{def:recall}, we obtain \textit{Recall(C, S)} $-$ $\frac{1}{m} \leq$ \textit{Recall(C, $\hat{S}$)} $\leq$ \textit{Recall(C, S)} $+$ $\frac{1}{m}$. We write this as $\lvert$\textit{Recall(C, S)} $-$ \textit{Recall(C, $\hat{S}$)}$\rvert$ $\leq \frac{1}{m}$.
\end{proof}

\vspace*{0.25cm}

\begin{corollary} \label{thm:recallerr2}
$\lvert$\textit{Recall(C, R)} $-$ \textit{Recall(C, $\hat{R}$)}$\rvert$ $\leq \frac{\epsilon}{m}$
\end{corollary}
\begin{proof}[Proof of Corollary \ref{thm:recallerr2}] \label{prf:recallerr2}
By Theorem \ref{thm:recall1err}, $\lvert$\textit{Recall(C, S)} $-$ \textit{Recall(C, $\hat{S}$)}$\rvert$ $\leq \frac{1}{m}$ for some arbitrary clustering $S$ and a second clustering $\hat{S}$ that is equivalent to $S$ but with a single data point belonging to a different cluster. Given a GTR $R$ and a corresponding AGTR $\hat{R}$, we can sequentially change the cluster membership of $\epsilon$ data points in $R$ to obtain $\hat{R}$. At each step the recall value can change by at most $\pm\frac{1}{m}$. Therefore, $\lvert$\textit{Recall(C, R)} $-$ \textit{Recall(C, $\hat{R}$)}$\rvert$ $\leq \frac{\epsilon}{m}$.
\end{proof}

\vspace*{0.25cm}

\begin{corollary}
$\lvert$\textit{Recall(C, D)} $-$ \textit{Recall(C, $\hat{R}$)}$\rvert$ $\leq \frac{\delta}{m}$
\end{corollary}
\begin{proof}[Proof of Corollary \ref{thm:recallerrbound}] \label{prf:recallerrbound}
By Theorem \ref{thm:recall1err}, \textit{Recall(C, S)} $-$ $\frac{1}{m} \leq$ \textit{Recall(C, $\hat{S}$)} $\leq$ \textit{Recall(C, S)} $+$ $\frac{1}{m}$ for some arbitrary clustering $S$ and a second clustering $\hat{S}$ that is equivalent to $S$ but with a single data point belonging to a different cluster. Given a ground truth reference label clustering $D$ and a corresponding AGTR $\hat{R}$,we can sequentially change the cluster membership of $\epsilon$ data points in $D$ to obtain $\hat{R}$. At each step the recall value can change by at most $\pm\frac{1}{m}$. Therefore, $\lvert$\textit{Recall(C, D)} $-$ \textit{Recall(C, $\hat{R}$)}$\rvert$ $\leq \frac{\delta}{m}$.
\end{proof}

\vspace*{0.25cm}

\begin{theorem}
If $\hat{\epsilon} \geq \epsilon$ then \textit{Precision(C, $\hat{R}$)} $-$ $\frac{\hat{\epsilon}}{m}$ $\leq$ \textit{Precision(C, D)}
\end{theorem}
\begin{proof} [Proof of Theorem \ref{thm:precisionagtr}] \label{prf:precisionagtr}
By \autoref{thm:precisionerr}, $\lvert$\textit{Precision(C, R)} $-$ \textit{Precision(C, $\hat{R}$)}$\rvert$ $\leq \frac{\epsilon}{m}$. This can be written as \textit{Precision(C, $\hat{R}$)} $-$ $\frac{\epsilon}{m}$ $\leq$ \textit{Precision(C, R)}. By applying \autoref{thm:precisiongtr}, \textit{Precision(C, $\hat{R}$)} $-$ $\frac{\hat{\epsilon}}{m}$ $\leq$ \textit{Precision(C, $\hat{R}$)} $-$ $\frac{\epsilon}{m}$ $\leq$ \textit{Precision(C, R)} $\leq$ \textit{Precision(C, D)}.
\end{proof}

\vspace*{0.25cm}

\begin{theorem}
If $\hat{\epsilon} \geq \epsilon$ then \textit{Recall(C, $\hat{R}$)} + $\frac{\hat{\epsilon}}{m}$ $\geq$ \textit{Recall(C, D)}
\end{theorem}
\begin{proof}[Proof of Theorem \ref{thm:recallagtr}] \label{prf:recallagtr}
By \autoref{thm:recallerr}, $\lvert$\textit{Recall(C, R)} $-$ \textit{Recall(C, $\hat{R}$)}$\rvert$ $\leq \frac{\epsilon}{m}$. This can be written as \textit{Recall(C, $\hat{R}$)} $+$ $\frac{\epsilon}{m}$ $\geq$ \textit{Recall(C, R)}. Therefore, by \autoref{thm:recallgtr}, \textit{Recall(C, $\hat{R}$)} $+$ $\frac{\hat{\epsilon}}{m}$ $\geq$ \textit{Recall(C, $\hat{R}$)} $+$ $\frac{\epsilon}{m}$ $\geq$ \textit{Recall(C, R)} $\geq$ \textit{Recall(C, D)}.
\end{proof}

\vspace*{0.25cm}

\begin{corollary}
If $\hat{\epsilon} \geq \epsilon$ then \textit{Recall(C, $\hat{R}$)} + $\frac{\hat{\epsilon}}{m}$ $\geq$ \textit{Accuracy(C, D)}
\end{corollary}
\begin{proof}[Proof of Corollary \ref{cor:accuracyagtr}]
By \autoref{thm:recallerr}, $\lvert$\textit{Recall(C, R)} $-$ \textit{Recall(C, $\hat{R}$)}$\rvert$ $\leq \frac{\epsilon}{m}$. This can be written as \textit{Recall(C, $\hat{R}$)} $+$ $\frac{\epsilon}{m}$ $\geq$ \textit{Recall(C, R)}. Using \autoref{cor:accuracygtr}, \textit{Recall(C, $\hat{R}$)} $+$ $\frac{\hat{\epsilon}}{m}$ $\geq$ \textit{Recall(C, $\hat{R}$)} $+$ $\frac{\epsilon}{m}$ $\geq$ \textit{Recall(C, R)} $\geq$ \textit{Accuracy(C, D)}.
\end{proof}

\end{appendices}

\end{document}

%% file: AGTR_all_diagram.tex
\tikzset{every picture/.style={line width=0.75pt}} %

\begin{tikzpicture}[x=0.75pt,y=0.75pt,yscale=-1,xscale=1]

\draw  [fill={rgb, 255:red, 218; green, 218; blue, 218 }  ,fill opacity=1 ] (50,45) .. controls (50,36.72) and (56.72,30) .. (65,30) .. controls (73.28,30) and (80,36.72) .. (80,45) .. controls (80,53.28) and (73.28,60) .. (65,60) .. controls (56.72,60) and (50,53.28) .. (50,45) -- cycle ;
\draw  [fill={rgb, 255:red, 218; green, 218; blue, 218 }  ,fill opacity=1 ] (110,45) .. controls (110,36.72) and (116.72,30) .. (125,30) .. controls (133.28,30) and (140,36.72) .. (140,45) .. controls (140,53.28) and (133.28,60) .. (125,60) .. controls (116.72,60) and (110,53.28) .. (110,45) -- cycle ;
\draw  [fill={rgb, 255:red, 218; green, 218; blue, 218 }  ,fill opacity=1 ] (50,105) .. controls (50,96.72) and (56.72,90) .. (65,90) .. controls (73.28,90) and (80,96.72) .. (80,105) .. controls (80,113.28) and (73.28,120) .. (65,120) .. controls (56.72,120) and (50,113.28) .. (50,105) -- cycle ;
\draw  [fill={rgb, 255:red, 218; green, 218; blue, 218 }  ,fill opacity=1 ] (110,105) .. controls (110,96.72) and (116.72,90) .. (125,90) .. controls (133.28,90) and (140,96.72) .. (140,105) .. controls (140,113.28) and (133.28,120) .. (125,120) .. controls (116.72,120) and (110,113.28) .. (110,105) -- cycle ;
\draw  [fill={rgb, 255:red, 218; green, 218; blue, 218 }  ,fill opacity=1 ] (50,165) .. controls (50,156.72) and (56.72,150) .. (65,150) .. controls (73.28,150) and (80,156.72) .. (80,165) .. controls (80,173.28) and (73.28,180) .. (65,180) .. controls (56.72,180) and (50,173.28) .. (50,165) -- cycle ;
\draw  [fill={rgb, 255:red, 218; green, 218; blue, 218 }  ,fill opacity=1 ] (110,165) .. controls (110,156.72) and (116.72,150) .. (125,150) .. controls (133.28,150) and (140,156.72) .. (140,165) .. controls (140,173.28) and (133.28,180) .. (125,180) .. controls (116.72,180) and (110,173.28) .. (110,165) -- cycle ;
\draw  [fill={rgb, 255:red, 218; green, 218; blue, 218 }  ,fill opacity=1 ] (50,225) .. controls (50,216.72) and (56.72,210) .. (65,210) .. controls (73.28,210) and (80,216.72) .. (80,225) .. controls (80,233.28) and (73.28,240) .. (65,240) .. controls (56.72,240) and (50,233.28) .. (50,225) -- cycle ;
\draw  [fill={rgb, 255:red, 218; green, 218; blue, 218 }  ,fill opacity=1 ] (110,225) .. controls (110,216.72) and (116.72,210) .. (125,210) .. controls (133.28,210) and (140,216.72) .. (140,225) .. controls (140,233.28) and (133.28,240) .. (125,240) .. controls (116.72,240) and (110,233.28) .. (110,225) -- cycle ;
\draw   (40,30) .. controls (40,24.48) and (44.48,20) .. (50,20) -- (80,20) .. controls (85.52,20) and (90,24.48) .. (90,30) -- (90,120) .. controls (90,125.52) and (85.52,130) .. (80,130) -- (50,130) .. controls (44.48,130) and (40,125.52) .. (40,120) -- cycle ;
\draw   (100,30) .. controls (100,24.48) and (104.48,20) .. (110,20) -- (140,20) .. controls (145.52,20) and (150,24.48) .. (150,30) -- (150,120) .. controls (150,125.52) and (145.52,130) .. (140,130) -- (110,130) .. controls (104.48,130) and (100,125.52) .. (100,120) -- cycle ;
\draw   (100,150) .. controls (100,144.48) and (104.48,140) .. (110,140) -- (140,140) .. controls (145.52,140) and (150,144.48) .. (150,150) -- (150,240) .. controls (150,245.52) and (145.52,250) .. (140,250) -- (110,250) .. controls (104.48,250) and (100,245.52) .. (100,240) -- cycle ;
\draw   (40,150) .. controls (40,144.48) and (44.48,140) .. (50,140) -- (80,140) .. controls (85.52,140) and (90,144.48) .. (90,150) -- (90,240) .. controls (90,245.52) and (85.52,250) .. (80,250) -- (50,250) .. controls (44.48,250) and (40,245.52) .. (40,240) -- cycle ;
\draw  [fill={rgb, 255:red, 218; green, 218; blue, 218 }  ,fill opacity=1 ] (240.5,45) .. controls (240.5,36.72) and (247.22,30) .. (255.5,30) .. controls (263.78,30) and (270.5,36.72) .. (270.5,45) .. controls (270.5,53.28) and (263.78,60) .. (255.5,60) .. controls (247.22,60) and (240.5,53.28) .. (240.5,45) -- cycle ;
\draw  [fill={rgb, 255:red, 218; green, 218; blue, 218 }  ,fill opacity=1 ] (300.5,45) .. controls (300.5,36.72) and (307.22,30) .. (315.5,30) .. controls (323.78,30) and (330.5,36.72) .. (330.5,45) .. controls (330.5,53.28) and (323.78,60) .. (315.5,60) .. controls (307.22,60) and (300.5,53.28) .. (300.5,45) -- cycle ;
\draw  [fill={rgb, 255:red, 218; green, 218; blue, 218 }  ,fill opacity=1 ] (240.5,105) .. controls (240.5,96.72) and (247.22,90) .. (255.5,90) .. controls (263.78,90) and (270.5,96.72) .. (270.5,105) .. controls (270.5,113.28) and (263.78,120) .. (255.5,120) .. controls (247.22,120) and (240.5,113.28) .. (240.5,105) -- cycle ;
\draw  [fill={rgb, 255:red, 218; green, 218; blue, 218 }  ,fill opacity=1 ] (300.5,105) .. controls (300.5,96.72) and (307.22,90) .. (315.5,90) .. controls (323.78,90) and (330.5,96.72) .. (330.5,105) .. controls (330.5,113.28) and (323.78,120) .. (315.5,120) .. controls (307.22,120) and (300.5,113.28) .. (300.5,105) -- cycle ;
\draw  [fill={rgb, 255:red, 218; green, 218; blue, 218 }  ,fill opacity=1 ] (240.5,165) .. controls (240.5,156.72) and (247.22,150) .. (255.5,150) .. controls (263.78,150) and (270.5,156.72) .. (270.5,165) .. controls (270.5,173.28) and (263.78,180) .. (255.5,180) .. controls (247.22,180) and (240.5,173.28) .. (240.5,165) -- cycle ;
\draw  [fill={rgb, 255:red, 218; green, 218; blue, 218 }  ,fill opacity=1 ] (300.5,165) .. controls (300.5,156.72) and (307.22,150) .. (315.5,150) .. controls (323.78,150) and (330.5,156.72) .. (330.5,165) .. controls (330.5,173.28) and (323.78,180) .. (315.5,180) .. controls (307.22,180) and (300.5,173.28) .. (300.5,165) -- cycle ;
\draw  [fill={rgb, 255:red, 218; green, 218; blue, 218 }  ,fill opacity=1 ] (240.5,225) .. controls (240.5,216.72) and (247.22,210) .. (255.5,210) .. controls (263.78,210) and (270.5,216.72) .. (270.5,225) .. controls (270.5,233.28) and (263.78,240) .. (255.5,240) .. controls (247.22,240) and (240.5,233.28) .. (240.5,225) -- cycle ;
\draw  [fill={rgb, 255:red, 218; green, 218; blue, 218 }  ,fill opacity=1 ] (300.5,225) .. controls (300.5,216.72) and (307.22,210) .. (315.5,210) .. controls (323.78,210) and (330.5,216.72) .. (330.5,225) .. controls (330.5,233.28) and (323.78,240) .. (315.5,240) .. controls (307.22,240) and (300.5,233.28) .. (300.5,225) -- cycle ;
\draw   (230.5,30) .. controls (230.5,24.48) and (234.98,20) .. (240.5,20) -- (330,20) .. controls (335.52,20) and (340,24.48) .. (340,30) -- (340,60) .. controls (340,65.52) and (335.52,70) .. (330,70) -- (240.5,70) .. controls (234.98,70) and (230.5,65.52) .. (230.5,60) -- cycle ;
\draw   (230,90) .. controls (230,84.48) and (234.48,80) .. (240,80) -- (330.5,80) .. controls (336.02,80) and (340.5,84.48) .. (340.5,90) -- (340.5,120) .. controls (340.5,125.52) and (336.02,130) .. (330.5,130) -- (240,130) .. controls (234.48,130) and (230,125.52) .. (230,120) -- cycle ;
\draw   (230.5,161.9) .. controls (230.5,149.8) and (240.3,140) .. (252.4,140) -- (318.1,140) .. controls (330.2,140) and (340,149.8) .. (340,161.9) -- (340,228.1) .. controls (340,240.2) and (330.2,250) .. (318.1,250) -- (252.4,250) .. controls (240.3,250) and (230.5,240.2) .. (230.5,228.1) -- cycle ;
\draw  [fill={rgb, 255:red, 218; green, 218; blue, 218 }  ,fill opacity=1 ] (400.5,45) .. controls (400.5,36.72) and (407.22,30) .. (415.5,30) .. controls (423.78,30) and (430.5,36.72) .. (430.5,45) .. controls (430.5,53.28) and (423.78,60) .. (415.5,60) .. controls (407.22,60) and (400.5,53.28) .. (400.5,45) -- cycle ;
\draw  [fill={rgb, 255:red, 218; green, 218; blue, 218 }  ,fill opacity=1 ] (460.5,45) .. controls (460.5,36.72) and (467.22,30) .. (475.5,30) .. controls (483.78,30) and (490.5,36.72) .. (490.5,45) .. controls (490.5,53.28) and (483.78,60) .. (475.5,60) .. controls (467.22,60) and (460.5,53.28) .. (460.5,45) -- cycle ;
\draw  [fill={rgb, 255:red, 218; green, 218; blue, 218 }  ,fill opacity=1 ] (400.5,105) .. controls (400.5,96.72) and (407.22,90) .. (415.5,90) .. controls (423.78,90) and (430.5,96.72) .. (430.5,105) .. controls (430.5,113.28) and (423.78,120) .. (415.5,120) .. controls (407.22,120) and (400.5,113.28) .. (400.5,105) -- cycle ;
\draw  [fill={rgb, 255:red, 218; green, 218; blue, 218 }  ,fill opacity=1 ] (460.5,105) .. controls (460.5,96.72) and (467.22,90) .. (475.5,90) .. controls (483.78,90) and (490.5,96.72) .. (490.5,105) .. controls (490.5,113.28) and (483.78,120) .. (475.5,120) .. controls (467.22,120) and (460.5,113.28) .. (460.5,105) -- cycle ;
\draw  [fill={rgb, 255:red, 218; green, 218; blue, 218 }  ,fill opacity=1 ] (400.5,165) .. controls (400.5,156.72) and (407.22,150) .. (415.5,150) .. controls (423.78,150) and (430.5,156.72) .. (430.5,165) .. controls (430.5,173.28) and (423.78,180) .. (415.5,180) .. controls (407.22,180) and (400.5,173.28) .. (400.5,165) -- cycle ;
\draw  [fill={rgb, 255:red, 218; green, 218; blue, 218 }  ,fill opacity=1 ] (460.5,165) .. controls (460.5,156.72) and (467.22,150) .. (475.5,150) .. controls (483.78,150) and (490.5,156.72) .. (490.5,165) .. controls (490.5,173.28) and (483.78,180) .. (475.5,180) .. controls (467.22,180) and (460.5,173.28) .. (460.5,165) -- cycle ;
\draw  [fill={rgb, 255:red, 218; green, 218; blue, 218 }  ,fill opacity=1 ] (400.5,225) .. controls (400.5,216.72) and (407.22,210) .. (415.5,210) .. controls (423.78,210) and (430.5,216.72) .. (430.5,225) .. controls (430.5,233.28) and (423.78,240) .. (415.5,240) .. controls (407.22,240) and (400.5,233.28) .. (400.5,225) -- cycle ;
\draw  [fill={rgb, 255:red, 218; green, 218; blue, 218 }  ,fill opacity=1 ] (460.5,225) .. controls (460.5,216.72) and (467.22,210) .. (475.5,210) .. controls (483.78,210) and (490.5,216.72) .. (490.5,225) .. controls (490.5,233.28) and (483.78,240) .. (475.5,240) .. controls (467.22,240) and (460.5,233.28) .. (460.5,225) -- cycle ;
\draw   (390.5,30) .. controls (390.5,24.48) and (394.98,20) .. (400.5,20) -- (430.5,20) .. controls (436.02,20) and (440.5,24.48) .. (440.5,30) -- (440.5,60) .. controls (440.5,65.52) and (436.02,70) .. (430.5,70) -- (400.5,70) .. controls (394.98,70) and (390.5,65.52) .. (390.5,60) -- cycle ;
\draw   (450.5,30) .. controls (450.5,24.48) and (454.98,20) .. (460.5,20) -- (490.5,20) .. controls (496.02,20) and (500.5,24.48) .. (500.5,30) -- (500.5,60) .. controls (500.5,65.52) and (496.02,70) .. (490.5,70) -- (460.5,70) .. controls (454.98,70) and (450.5,65.52) .. (450.5,60) -- cycle ;
\draw   (450.5,150) .. controls (450.5,144.48) and (454.98,140) .. (460.5,140) -- (490.5,140) .. controls (496.02,140) and (500.5,144.48) .. (500.5,150) -- (500.5,240) .. controls (500.5,245.52) and (496.02,250) .. (490.5,250) -- (460.5,250) .. controls (454.98,250) and (450.5,245.52) .. (450.5,240) -- cycle ;
\draw   (390.5,150) .. controls (390.5,144.48) and (394.98,140) .. (400.5,140) -- (430.5,140) .. controls (436.02,140) and (440.5,144.48) .. (440.5,150) -- (440.5,240) .. controls (440.5,245.52) and (436.02,250) .. (430.5,250) -- (400.5,250) .. controls (394.98,250) and (390.5,245.52) .. (390.5,240) -- cycle ;
\draw   (390.5,90) .. controls (390.5,84.48) and (394.98,80) .. (400.5,80) -- (490,80) .. controls (495.52,80) and (500,84.48) .. (500,90) -- (500,120) .. controls (500,125.52) and (495.52,130) .. (490,130) -- (400.5,130) .. controls (394.98,130) and (390.5,125.52) .. (390.5,120) -- cycle ;
\draw  [fill={rgb, 255:red, 218; green, 218; blue, 218 }  ,fill opacity=1 ] (590.5,45) .. controls (590.5,36.72) and (597.22,30) .. (605.5,30) .. controls (613.78,30) and (620.5,36.72) .. (620.5,45) .. controls (620.5,53.28) and (613.78,60) .. (605.5,60) .. controls (597.22,60) and (590.5,53.28) .. (590.5,45) -- cycle ;
\draw  [fill={rgb, 255:red, 218; green, 218; blue, 218 }  ,fill opacity=1 ] (650.5,45) .. controls (650.5,36.72) and (657.22,30) .. (665.5,30) .. controls (673.78,30) and (680.5,36.72) .. (680.5,45) .. controls (680.5,53.28) and (673.78,60) .. (665.5,60) .. controls (657.22,60) and (650.5,53.28) .. (650.5,45) -- cycle ;
\draw  [fill={rgb, 255:red, 218; green, 218; blue, 218 }  ,fill opacity=1 ] (590.5,105) .. controls (590.5,96.72) and (597.22,90) .. (605.5,90) .. controls (613.78,90) and (620.5,96.72) .. (620.5,105) .. controls (620.5,113.28) and (613.78,120) .. (605.5,120) .. controls (597.22,120) and (590.5,113.28) .. (590.5,105) -- cycle ;
\draw  [fill={rgb, 255:red, 218; green, 218; blue, 218 }  ,fill opacity=1 ] (650.5,105) .. controls (650.5,96.72) and (657.22,90) .. (665.5,90) .. controls (673.78,90) and (680.5,96.72) .. (680.5,105) .. controls (680.5,113.28) and (673.78,120) .. (665.5,120) .. controls (657.22,120) and (650.5,113.28) .. (650.5,105) -- cycle ;
\draw  [fill={rgb, 255:red, 218; green, 218; blue, 218 }  ,fill opacity=1 ] (590.5,165) .. controls (590.5,156.72) and (597.22,150) .. (605.5,150) .. controls (613.78,150) and (620.5,156.72) .. (620.5,165) .. controls (620.5,173.28) and (613.78,180) .. (605.5,180) .. controls (597.22,180) and (590.5,173.28) .. (590.5,165) -- cycle ;
\draw  [fill={rgb, 255:red, 218; green, 218; blue, 218 }  ,fill opacity=1 ] (650.5,165) .. controls (650.5,156.72) and (657.22,150) .. (665.5,150) .. controls (673.78,150) and (680.5,156.72) .. (680.5,165) .. controls (680.5,173.28) and (673.78,180) .. (665.5,180) .. controls (657.22,180) and (650.5,173.28) .. (650.5,165) -- cycle ;
\draw  [fill={rgb, 255:red, 218; green, 218; blue, 218 }  ,fill opacity=1 ] (590.5,225) .. controls (590.5,216.72) and (597.22,210) .. (605.5,210) .. controls (613.78,210) and (620.5,216.72) .. (620.5,225) .. controls (620.5,233.28) and (613.78,240) .. (605.5,240) .. controls (597.22,240) and (590.5,233.28) .. (590.5,225) -- cycle ;
\draw  [fill={rgb, 255:red, 218; green, 218; blue, 218 }  ,fill opacity=1 ] (650.5,225) .. controls (650.5,216.72) and (657.22,210) .. (665.5,210) .. controls (673.78,210) and (680.5,216.72) .. (680.5,225) .. controls (680.5,233.28) and (673.78,240) .. (665.5,240) .. controls (657.22,240) and (650.5,233.28) .. (650.5,225) -- cycle ;
\draw   (580.5,30) .. controls (580.5,24.48) and (584.98,20) .. (590.5,20) -- (620.5,20) .. controls (626.02,20) and (630.5,24.48) .. (630.5,30) -- (630.5,60) .. controls (630.5,65.52) and (626.02,70) .. (620.5,70) -- (590.5,70) .. controls (584.98,70) and (580.5,65.52) .. (580.5,60) -- cycle ;
\draw   (640.5,30) .. controls (640.5,24.48) and (644.98,20) .. (650.5,20) -- (680.5,20) .. controls (686.02,20) and (690.5,24.48) .. (690.5,30) -- (690.5,60) .. controls (690.5,65.52) and (686.02,70) .. (680.5,70) -- (650.5,70) .. controls (644.98,70) and (640.5,65.52) .. (640.5,60) -- cycle ;
\draw   (640.5,150) .. controls (640.5,144.48) and (644.98,140) .. (650.5,140) -- (680.5,140) .. controls (686.02,140) and (690.5,144.48) .. (690.5,150) -- (690.5,240) .. controls (690.5,245.52) and (686.02,250) .. (680.5,250) -- (650.5,250) .. controls (644.98,250) and (640.5,245.52) .. (640.5,240) -- cycle ;
\draw   (580.5,150) .. controls (580.5,144.48) and (584.98,140) .. (590.5,140) -- (620.5,140) .. controls (626.02,140) and (630.5,144.48) .. (630.5,150) -- (630.5,240) .. controls (630.5,245.52) and (626.02,250) .. (620.5,250) -- (590.5,250) .. controls (584.98,250) and (580.5,245.52) .. (580.5,240) -- cycle ;
\draw   (580.5,90) .. controls (580.5,84.48) and (584.98,80) .. (590.5,80) -- (680,80) .. controls (685.52,80) and (690,84.48) .. (690,90) -- (690,120) .. controls (690,125.52) and (685.52,130) .. (680,130) -- (590.5,130) .. controls (584.98,130) and (580.5,125.52) .. (580.5,120) -- cycle ;
\draw  [line width=2.25]  (635.5,26.24) .. controls (635.5,19.7) and (640.8,14.4) .. (647.34,14.4) -- (682.86,14.4) .. controls (689.4,14.4) and (694.7,19.7) .. (694.7,26.24) -- (694.7,123.36) .. controls (694.7,129.9) and (689.4,135.2) .. (682.86,135.2) -- (647.34,135.2) .. controls (640.8,135.2) and (635.5,129.9) .. (635.5,123.36) -- cycle ;
\draw  [dash pattern={on 0.84pt off 2.51pt}]  (190,10) -- (190,280) ;
\draw  [dash pattern={on 0.84pt off 2.51pt}]  (350,10) -- (350,280) ;
\draw  [dash pattern={on 0.84pt off 2.51pt}]  (540,10) -- (540,280) ;

\draw (65,45) node  [font=\large]  {$1$};
\draw (125,45) node  [font=\large]  {$2$};
\draw (65,105) node  [font=\large]  {$3$};
\draw (125,105) node  [font=\large]  {$4$};
\draw (65,165) node  [font=\large]  {$5$};
\draw (125,165) node  [font=\large]  {$6$};
\draw (65,225) node  [font=\large]  {$7$};
\draw (125,225) node  [font=\large]  {$8$};
\draw (21,80) node  [font=\LARGE]  {$C_{1}$};
\draw (21.5,199) node  [font=\LARGE]  {$C_{3}$};
\draw (168.5,201) node  [font=\LARGE]  {$C_{4}$};
\draw (171.5,81) node  [font=\LARGE]  {$C_{2}$};
\draw (91.5,269) node  [font=\Huge]  {$\boldsymbol{C}$};
\draw (255.5,45) node  [font=\large]  {$1$};
\draw (315.5,45) node  [font=\large]  {$2$};
\draw (255.5,105) node  [font=\large]  {$3$};
\draw (315.5,105) node  [font=\large]  {$4$};
\draw (255.5,165) node  [font=\large]  {$5$};
\draw (315.5,165) node  [font=\large]  {$6$};
\draw (255.5,225) node  [font=\large]  {$7$};
\draw (315.5,225) node  [font=\large]  {$8$};
\draw (212.5,49) node  [font=\LARGE]  {$D_{1}$};
\draw (211.5,199) node  [font=\LARGE]  {$D_{3}$};
\draw (212.5,101) node  [font=\LARGE]  {$D_{2}$};
\draw (282,269) node  [font=\Huge]  {$\boldsymbol{D}$};
\draw (415.5,45) node  [font=\large]  {$1$};
\draw (475.5,45) node  [font=\large]  {$2$};
\draw (415.5,105) node  [font=\large]  {$3$};
\draw (475.5,105) node  [font=\large]  {$4$};
\draw (415.5,165) node  [font=\large]  {$5$};
\draw (475.5,165) node  [font=\large]  {$6$};
\draw (415.5,225) node  [font=\large]  {$7$};
\draw (475.5,225) node  [font=\large]  {$8$};
\draw (371.5,101) node  [font=\LARGE]  {$R_{3}$};
\draw (372,199) node  [font=\LARGE]  {$R_{4}$};
\draw (519,201) node  [font=\LARGE]  {$R_{5}$};
\draw (521.5,49) node  [font=\LARGE]  {$R_{2}$};
\draw (442,269) node  [font=\Huge]  {$\boldsymbol{R}$};
\draw (371.5,49) node  [font=\LARGE]  {$R_{1}$};
\draw (605.5,45) node  [font=\large]  {$1$};
\draw (665.5,45) node  [font=\large]  {$2$};
\draw (605.5,105) node  [font=\large]  {$3$};
\draw (665.5,105) node  [font=\large]  {$4$};
\draw (605.5,165) node  [font=\large]  {$5$};
\draw (665.5,165) node  [font=\large]  {$6$};
\draw (605.5,225) node  [font=\large]  {$7$};
\draw (665.5,225) node  [font=\large]  {$8$};
\draw (561.5,101) node  [font=\LARGE]  {$\hat{R}_{2}$};
\draw (562,199) node  [font=\LARGE]  {$\hat{R}_{3}$};
\draw (709,201) node  [font=\LARGE]  {$\hat{R}_{4}$};
\draw (632,269) node  [font=\Huge]  {$\hat{\boldsymbol{R}}$};
\draw (561.5,49) node  [font=\LARGE]  {$\hat{R}_{1}$};

\end{tikzpicture}